\documentclass[11pt]{article}

\usepackage{fullpage}
\usepackage{hyperref}       
\usepackage{url}            
\usepackage{booktabs}       
\usepackage{amsfonts}       
\usepackage{nicefrac}       
\usepackage{microtype}      

\usepackage{algorithm}
\usepackage{algpseudocode}
\usepackage{natbib}

\usepackage{color}
\usepackage[dvipsnames]{xcolor}

\usepackage{amsmath}
\usepackage{amsthm}
\usepackage{amssymb}
\usepackage{enumitem}

\newtheorem{theorem}{Theorem}[section]
\newtheorem{lemma}[theorem]{Lemma}
\newtheorem{definition}[theorem]{Definition}

\newtheorem{proposition}[theorem]{Proposition}
\newtheorem{corollary}[theorem]{Corollary}

\newcommand{\wh}{\widehat}
\newcommand{\wt}{\widetilde}

\newcommand{\bepsilon}{\boldsymbol{\epsilon}}

\newcommand{\cH}{\mathcal{H}}

\renewcommand{\varepsilon}{\epsilon}

\DeclareMathOperator{\aux}{aux}

\DeclareMathOperator{\poly}{poly}

\DeclareMathOperator{\one}{\boldsymbol{1}}

\DeclareMathOperator{\clip}{clip}

\newcommand{\RR}{\mathbb{R}}

\newcommand{\cT}{\mathcal{T}}
\newcommand{\cS}{\mathcal{S}}

\newcommand{\cA}{\mathcal{A}}
\newcommand{\ms}{\sigma^{\infty}}
\newcommand{\mpi}[1]{\pi^{\infty #1}}

\newcommand{\cM}{\mathcal{M}}
\newcommand{\EE}{\mathbb{E}}
\newcommand{\cE}{\mathcal{E}}

\newcommand{\bw}{\boldsymbol{w}}

\newcommand{\bsigma}{\boldsymbol{\sigma}}

\newcommand{\bxi}{\boldsymbol{\xi}}

\newcommand{\cG}{\mathcal{G}}
\newcommand{\bg}{\boldsymbol{g}}
\newcommand{\bQ}{\boldsymbol{Q}}

\newcommand{\bP}{\boldsymbol{P}}

\newcommand{\br}{\boldsymbol{r}}

\newcommand{\bv}{\boldsymbol{v}}

\newcommand{\var}{\mathrm{var}}
\newcommand{\defeq}{:=}

\newcommand{\strat}{\sigma}

\newcommand{\statemin}{\cS_{\min}}
\newcommand{\statemax}{\cS_{\max}}
\newcommand{\states}{\cS}

\newcommand{\actions}{\cA}

\newcommand{\otilde}{\wt{O}}

\title{\huge Solving Discounted Stochastic Two-Player Games\\
with Near-Optimal Time and Sample Complexity}

\author{%
	Aaron Sidford\thanks{Supported by NSF CAREER Award CCF-184485} \\
	Stanford University\\
	\texttt{sidford@stanford.edu}
	\and
	Mengdi Wang\\
	Princeton University\\
	\texttt{mengdiw@princeton.edu}
	\and
	Lin F. Yang\\
	University of California, Los Angeles\\
	\texttt{lyang36@g.ucla.edu}
	\and
	Yinyu Ye\\
	Stanford University\\
	\texttt{yyye@stanford.edu}
}

\begin{document}

\maketitle

\begin{abstract}
In this paper we settle  the sampling complexity of solving discounted two-player turn-based zero-sum stochastic games up to polylogarithmic factors. Given a stochastic game 
with discount factor $\gamma\in(0,1)$ we provide an algorithm that computes an $\epsilon$-optimal strategy with high-probability 
given $\otilde((1 - \gamma)^{-3} \epsilon^{-2})$ samples from the transition function for each state-action-pair. Our algorithm runs in time nearly linear in the number of samples and uses space nearly linear in the number of state-action pairs. As stochastic games generalize Markov decision processes (MDPs) our runtime and sample complexities are optimal due to \cite{azar2013minimax}. We achieve our results by showing how to generalize a near-optimal Q-learning based algorithms for MDP, 
 in particular
 \cite{sidford2018near},  to two-player strategy computation algorithms. This overcomes limitations of standard Q-learning and strategy iteration or alternating minimization based approaches and we hope will pave the way for future reinforcement learning results by facilitating the extension of MDP results to multi-agent settings with little loss. 
\end{abstract}

\clearpage

\section{Introduction}
\label{sec:intro}

In this paper we study the sample complexity of learning a near-optimal strategy in discounted two-player turn-based zero-sum stochastic games \cite{shapley1953stochastic,hansen2013strategy}, which we refer to more concisely as {\it stochastic games}. Stochastic games model dynamic strategic settings in which two players take turns and the state of game evolves stochastically according to some transition law. 
This model 
 encapsulates a major challenge in multi-agent learning: other agents may be learning and adapting as well.
Further, stochastic games are a generalization of the Markov decision process (MDP), a fundamental model for reinforcement learning, to the two-player setting \cite{littman1994markov}. 
MDPs can be viewed as degenerate stochastic games in which one of the players has no influence.
Consequently, understanding stochastic games is a natural step towards resolving challenges in reinforcement learning of extending single-agent learning to multi-agent settings. 

There is a long line of research in both MDPs and stochastic games (for a more thorough introduction, see \cite{filar2012competitive, hansen2013strategy} and references therein). Strikingly, \cite{hansen2013strategy} showed that there exists a pure-strategy Nash equilibrium which can be computed in strongly polynomial time for stochastic games, if the game matrix is fully accessible and the discount factor is fixed. In reinforcement learning settings, however, the transition function of the game is unknown and a common goal is to obtain an approximately optimal strategy (a function that maps states to actions) that is able to obtain an expected cumulative reward of at least (or at most) the Nash equilibrium value no matter what the other player does. Unfortunately, despite interest in generalizing MDP results to stochastic games, currently the best known running times/sample complexity for solving stochastic games in a variety of settings are worse than for solving MDPs. This may not be surprising since in general stochastic games are harder to solve than MDPs, e.g., whereas MDPS can be solved in (weakly) polynomial time it remains open whether or not the same can be done for 
stochastic games.


There are two natural approaches towards achieving sample complexity bounds for solving stochastic games. The first is to note that the popular stochastic value iteration, dynamic programming, and Q-learning methods all apply to stochastic games \cite{littman1994markov,hu2003nash,littman2001friend,perolat2015approximate}. Consequently, recent advances in these methods 
\cite{kearns1999finite,  sidford2018variance} developed for MDPs can be
directly generalized to solving stochastic games (though the sample complexity of these generalized methods has not been analyzed previously). It is tempting to generalize the analysis of sample optimal methods for estimating values \cite{azar2013minimax} and estimating policies \cite{sidford2018near} of MDPs to stochastic games. However, this is challenging as these methods rely on monotonicities in MDPs induced by the linear program nature of the problem \cite{azar2013minimax, sidford2018near}.

 The second approach would be to apply strategy iteration or alternating minimization / maximization to reduce solving stochastic games to approximately solving a sequence of MDPs. Unfortunately, the best analysis of such a method  \cite{hansen2013strategy} requires solving $\Omega(1/(1- \gamma))$ MDPs. Consequently, even if this approach could be carried out with approximate MDP solvers, the resulting sample complexity for solving stochastic games would be larger than that needed for solving MDPs. More discussion of related literatures is given in Section~\ref{sec:prev_work}.

Given the importance of solving stochastic games in reinforcement learning (e.g. \cite{hu1998multiagent, bowling2000analysis, bowling2001rational, hu2003nash, arslan2017decentralized}), this suggests the following fundamental open problem:

\emph{Can we design stochastic game learning algorithms that provably match the performance of MDP algorithms and achieve near-optimal sample complexities?}

In this paper, we answer this question in the affirmative in the particular case of solving discounted stochastic games with a generative model, i.e. an oracle for sampling from the transition function for state-action pairs. We provide an algorithm with the same near-optimal sample complexity that is known for solving discounted MDPs. Further, we achieve this result by showing how to transform particular MDP algorithms to solving stochastic games that satisfy particular two-sided monotonicity constraints. 
Therefore, while there is a major gap between MDPs and stochastic games in terms of computation time for obtaining the exact solutions, this gap disappears when considering the sampling complexity between the two.
We hope this work opens the door to more generally extend results for MDP to stochastic games and thereby enable the application of the rich research on reinforcement learning to a broader multi-player settings with little overhead.

\subsection{The Model}

Formally, throughout this paper, we consider \emph{discounted turn-based two-player zero-sum stochastic games} described as the tuple $\cG=(\cS_{\min}, \cS_{\max}, \cA, \bP, \br, \gamma)$. In these games there are two players, a \emph{min} or \emph{minimization} player which seeks to minimize the cumulative reward in the game and a \emph{max} or \emph{maximization} player which seeks to maximize the cumulative reward. Here, $\cS_{\min}$ and $\cS_{\max}$ are disjoint finite sets of \emph{states} controlled by the min-player and the max-player respectively and their union $\cS \defeq \cS_{\min} \cup \cS_{\max}$ is the set of all possible \emph{states of the game}. Further, $\cA$ is a finite set of \emph{actions} available at each state, $\bP: \cS \times  \cA \times \cS \mapsto [0,1]$ is a \emph{transition probability function}, $\br: \cS \times\cA  \mapsto [0,1]$ is the payoff or \emph{reward function} and $\gamma\in(0,1)$ is a discount factor.\footnote{Standard reductions allow this result to be applied for rewards of a broader range \cite{sidford2018near}. Further, while we assume there are the same number of actions per-state, our results easily extend to the case where this is non-uniform; in this case our dependencies on $|\cS| |\cA|$ can be replaced with the number of state-action pairs.}

Stochastic games $\cG=(\cS_{\min}, \cS_{\max}, \cA, \bP, \br, \gamma)$ are played dynamically in a sequence of turns, $\{t\}_{t=0}^{\infty}$, starting from some initial state $s^0 \in \cS$ at turn $t = 0$. In each turn $t \geq 0$, the game is in one of the states $s^t \in \cS$ and the player who controls the state $s^t$  chooses or \emph{plays} an 
action $a^t$ from the action space $\cA$. This action yields reward $r^t \defeq r(s^t, a^t)$ for the turn and causes the next state $s^{t+1}$ to be chosen at random from $\states$ where the transition probability $\Pr[s^{t + 1} = s' | s_1,...,s_t,a_1,...,a_t] = \bP(s'~|~s^t, a^t)$. 
The goal of the min-player (resp. max-player) is to choose actions to minimize (resp. maximize) the expected infinite-horizon discounted-reward or \emph{value} of the game $\sum_{t = 0}^{\infty} \gamma^t r^t$.
 
In this paper we focus on the case where the players play pure (deterministic) stationary strategies (policies), i.e. strategies which depend only on the current state. That is we wish to compute a \emph{min-player strategy} or \emph{policy}
$\pi_{\min}:\cS_{\min}\rightarrow \cA$ which defines the action the min player chooses at a state in $\cS_{\min}$ and \emph{max-player strategy} $\pi_{\max}:\cS_{\max}\rightarrow \cA$ which defines the action the max player chooses at a state in $\cS_{\max}$.
We call a pair of min-player and max-player strategies $\sigma = (\pi_{\min}, \pi_{\max})$ simply a \emph{strategy}.  Further, we let $\sigma(s) \defeq \pi_{\min}(s)$ for $s\in \cS_{\min}$ and $\sigma(s) \defeq \pi_{\max}(s)$ for $s\in \cS_{\max}$ and define the  \emph{value function} or \emph{expected discounted cumulative reward} by $\bv^{\sigma}$ where
\[
\bv^{\sigma}(s) = \bv[\sigma](s):= \EE\bigg[\sum_{t=0}^{\infty}\gamma^t\br(s^t, \sigma(s^t))~\Big|~s^0 = s\bigg] \quad \text{for all } s \in \cS
\]
and the expectation is over the random sequence of states, $s^0, s^1, s^2, \ldots $ generated according to $\bP$ under the strategy $\sigma$, i.e. $\Pr[s^{t + 1} = s'~|~s^t, s^{t - 1}, .... ,s^{0}] = \bP(s'~|~s^t, \strat(s^t))$ for all $t > 0$. 

Our goal in solving a game is to compute an approximate {\it Nash equilibrium} restricted to stationary strategies \cite{nash1951non,marskin2001markov}. 
We call a strategy $\sigma = (\pi_{\min}, \pi_{\max})$ an {\it equilibrium strategy} or \emph{optimal} if
\[
\max_{\pi_{\max}' : \statemax \rightarrow \actions} \bv^{(\pi_{\min},\pi_{\max}')}  \leq 
 \bv^{\sigma} \leq  \min_{\pi_{\min}' : \statemin \rightarrow \actions} \bv^{(\pi_{\min}',\pi_{\max})}.
\] 
and we call it $\epsilon$-optimal if these same inequalities hold up to an additive $\epsilon$ entrywise.
It is worth noting that the best response strategy to a stationary policy is also stationary \cite{fudenberg1991game} and
there always exists a pure stationary strategy attaining the Nash equilibrium  \cite{shapley1953stochastic}. Consequently, it is sufficient to focus on deterministic strategies. 



Throughout this paper we focus on solving stochastic games in the learning setting where the game is not fully specified. We assume that a \emph{generative model} is available which given any state-action pair, i.e. $s \in \states$ and $a \in \actions$, can sample a random $s'$ independently at random from the transition probability function, i.e. $\Pr[s' = t] = \bP(t ~|~s, a)$. Accessibility to a generative model is a standard and natural assumption (\cite{kakade2003sample,  azar2013minimax, sidford2018near, agarwal2019optimality}) and corresponds to PAC learning. 
The special case of solving a MDP given a generative model has been studied extensively (\cite{kakade2003sample, azar2013minimax, sidford2018variance, sidford2018near, agarwal2019optimality}) and is a natural proving ground towards designing theoretically motivated reinforcement learning algorithms.

\subsection{Our Results}


In this paper 
we provide an algorithm that computes an  $\epsilon$-optimal strategy using a sample size that matches the best known sample complexity for solving discounted MDPs.
Further, our algorithm runs in time proportional to the number of samples and space proportional to $|\cS||\cA|$. Interestingly, we achieve this result by showing how to run two-player variant of Q-learning such that the value-strategy sequences induced enjoy certain monotonicity properties. Essentially, we show that provided a value improving algorithm is sufficiently stable, then it can be extended to the two-player setting with limited loss.
This allows us to leverage recent advances in solving single player games to solve stochastic games with limited overhead. Our main result is given below.


\begin{theorem}[Main Theorem]
	\label{mainthm}
	There is an algorithm which given a stochastic game, $\cG =(\cS_{\min}, \cS_{\max}, \bP, \br, \gamma)$ with a generative model, outputs, with probability at least $1-\delta$,
	an $\epsilon$-optimal strategy $\sigma$ by querying 
	$Z=\wt{O}(|\cS| |\cA| (1-\gamma)^{-3} \epsilon^{-2})$
	samples, where $\epsilon\in(0,1)$ and $\wt{O}(\cdot)$ hides polylogarithmic factors. The algorithm runs in time $O(Z)$ and uses space $O(|\cS||\cA|)$.
\end{theorem}

Our sample and time complexities are optimal due to a known lower bound in the single player case by \cite{azar2013minimax}. It was shown in \cite{azar2013minimax} that solving any one-player MDP to $\epsilon$-optimality with high probability needs at least $\Omega(|\cS||\cA|(1-\gamma)^{-3} \epsilon^{-2})$ samples. 
Our sample complexity upper bound generalizes the recent sharp sample complexity results for solving the discounted MDP \cite{sidford2018near, agarwal2019optimality}, and tightly matches the information-theoretic sample complexity up to polylogarithmic factors. This result provides the first and near-optimal sample complexity for solving the two-person stochastic game.

\subsection{Notations and Preliminaries}
\label{sec:prelim}

\paragraph{Notation:} 
We use $\one$ to denote the all-ones vector whose dimension is adapted to the context.
We use the operators $|\cdot|, (\cdot)^2, \sqrt{\cdot}, \le, \ge$ as entrywise operators on vectors.
We identify the transition probability function $\bP$ as a matrix in $\RR^{(\cS\times\cA)\times \cS}$ and each row $\bP(\cdot~|~s,a)\in \RR^{\cS}$ as a vector.
We denote $\bv$ as a vector in $\RR^{\cS}$ and $\bQ$ as a vector in $\RR^{\cS\times \cA}$.
Therefore $\bP\bv$ is a vector in $\RR^{\cS\times \cA}$.
We use $\sigma$ to denote strategy pairs and $\pi$ for the min-player or max-player strategy.
For any strategy $\sigma$, we define $\bQ_{\sigma}\in \RR^{\cS}$ as $\bQ_{\sigma}(s) \defeq \bQ(s,\sigma(s))$ for $\forall s\in \cS$.
We denote $\bP^{\sigma}$ as a linear operator defined as
\[
\forall s\in \cS:\quad
[\bP^{\sigma} \bv](s) = \bP(\cdot~|~s,\sigma(s))^\top \bv, \quad\text{and}\quad \forall s,a\in \cS\times \cA:
[\bP^{\sigma} \bQ](s,a) = \bP(\cdot~|~s,a)^\top  \bQ_{\sigma}.
\]


\paragraph{Min-value and max-value:} For a min-player strategy $\pi_{\min}$, we define its \emph{value} as
\begin{align}
\label{value:pimin}
\bv^{\pi_{\min}} := \max_{\pi_{\max} ~ : ~ \statemax \rightarrow \actions} \bv^{(\pi_{\min}, \pi_{\max})},
\end{align}
We let $\sigma_{\max}(\pi_{\min})$ denote a maximizing argument of the above and call it an \emph{optimal counter strategy} of $\pi_{\min}$. 
Thus a value of a min-player strategy gives his expected reward in the \emph{worst case}.
We say a min-player strategy $\pi_{\min}$ is \emph{$\epsilon$-optimal} if 
$$\bv^{\pi_{\min}} \le  \min_{\pi_{\min}' ~ : ~ \statemin \rightarrow \actions}\bv^{\pi_{\min}'} +\epsilon\cdot\one,
\quad \text{entrywisely}.
$$
The value and $\epsilon$-optimality for the max player is defined similarly. We denote by $\sigma^*$ the optimal strategy and by $\bv^*$ the value function of the optimal strategy.



\paragraph{$Q$-function:} For a strategy  $\sigma$, we denote its \emph{$Q$-function} 
(or \emph{action value}) as $\bQ^{\sigma}\in \RR^{\cS\times \cA}$ by 
$ \bQ^\sigma := \br + \gamma \bP \bv^{\sigma}.$ 
For a vector $\bv \in \RR^{\cS}$ we denote $\bQ(\bv):=\br + \gamma \bP \bv$.
Given a $\bQ \in  \RR^{\cS\times \cA}$, we denote the greedy value of $\bQ$ as
\[
V[\bQ] (s):= \min_{a\in \cA}~ \bQ(s,a)\quad\text{if} \quad s\in \cS_{\min}  \quad\text{and}\quad V[\bQ] (s):= \max_{a\in \cA}\quad \bQ(s,a)\quad\text{if}\quad s\in \cS_{\max}.
\]
 
\paragraph{Bellman Operator:} 
We denote the Bellman operator, $\cT$, as follows: $\cT[\bv]\in \RR^{\cS}$, 
and  
\begin{align*}
\cT[\bv](s) &:= V[\br+\gamma\bP \bv].
\end{align*}
We also denote the greedy strategy, $\sigma(\bv)$ or $\sigma(\bQ)$, as the maximization/minimization argument of the $\cT$ operator. Moreover, for a given strategy $\sigma$, we denote $\cT_{\sigma}[\bv] = \bQ(\bv)_{\sigma}$.
For a given min-player strategy $\pi_{\min}$, we define the \emph{half} Bellman operator $\cH_{\pi_{\min}}$ 
\[
\cH_{\pi_{\min}} [\bv] = \br(s,\pi_{\min}(s)) +\gamma \bP(\cdot~|~s,\pi_{\min}(s))^\top \bv \quad\text{if} \quad s\in \cS_{\min};  \qquad \cT[\bv](s)\quad\text{if}\quad s\in \cS_{\max}.
\]
We define $\cH_{\pi_{\max}}$ similarly.
Note that $\bv^*$ is  the unique fixed point of the Bellman operator, i.e., $\cT [\bv^*]= \bv^*$ (known as the Bellman equation \cite{bellman1957dynamic}).
Similarly, $\bv^{\pi_{\min}}$ (resp. $\bv^{\pi_{\max}}$) is the unique fixed point for $\cH_{\pi_{\min}}$ (resp. $\cH_{\pi_{\max}}$).
The (half) Bellman-operators satisfy the following properties (see. e.g. \cite{hansen2013strategy,puterman2014markov})
\begin{enumerate}
\item \emph{contraction}: $\|\cT [\bv_1] - \cT [\bv_2]\|_{\infty} \le \gamma \| \bv_1 -  \bv_2\|_{\infty}$;
\item \emph{monotonicity}: $\bv_1\le \bv_2 \Rightarrow \cT [\bv_1] \le \cT [\bv_2]$.
\end{enumerate} 

\paragraph{High Probability:} we say an algorithm has a property ``with high probability'' if for any $\delta$ by increasing the time and sample complexity by $O(\log(1/\delta))$ it has the property with probability $1 - \delta$. 

\subsection{Previous Work}
\label{sec:prev_work}

Here we provide a more detailed survey of previous works related to stochastic games and MDPs.
Two-person stochastic games generalize 
MDPs \cite{shapley1953stochastic}.  When one of the players has only one action to choose from, the problem reduces to a MDP. 
A related game is the stochastic game where both players choose their respective actions simultaneously at each state and the process transitions to the next state under the control of both players  \cite{shapley1953stochastic}.  The turn-based stochastic game can be reduced to the game with simultaneous moves \cite{prolat2015approximateDP}.


Computing an optimal strategy for a two-player turn-based zero-sum stochastic game is known to be in NP $\cap$ co-NP \cite{condon1992complexity}. 
Later \cite{hansen2013strategy} showed that the strategy iteration, a generalization of Howard's policy iteration algorithm \cite{howard1960dynamic}, solves the discounted problem in strongly polynomial time when the discount factor is fixed. Their work uses ideas from \cite{ye2011simplex} which proved that the policy iteration algorithm solves the discounted MDP (DMDP) in strongly polynomial time when the discount factor is fixed. In general (e.g., if the discount factor is part of the input size), it is open if stochastic games can be solved in polynomial time \cite{littman1996algorithms}. This is in contrast to MDPs which can be solved in (weakly) polynomial time as they are a special case of linear programming.



The algorithms and complexity theory for solving two-player stochastic games is closely related to that of solving MDPs. Their is vast literature on solving MDPs which dates back to Bellman who developed value iteration in 1957 \cite{bellman1957dynamic}. The policy iteration was introduced shortly after by Howard \cite{howard1960dynamic}, and its complexity has been extensive studied in \cite{mansour1999complexity,ye2011simplex,scherrer2013improved}. Then \cite{d1963probabilistic} and \cite{de1960problemes} discovered that MDPs are special cases 
of a linear program, which leads to the insight that the simplex method, when applied to solving DMDPs, is a 
simple policy iteration method. Ye \cite{ye2011simplex} showed that policy iteration (which is a variant of the general simplex method for linear programming) and the simplex method are strongly polynomial for DMDP and terminate in $O(|\cS|^2|\cA| (1-\gamma)^{-1} \log( |\cS| (1-\gamma)^{-1}))$ iterations. 
\cite{hansen2013strategy} and \cite{scherrer2013improved} improved the iteration bound to $O(|\cS||\cA| (1-\gamma)^{-1} \log(|\cS|(1-\gamma)^{-1}))$ 
for Howard's policy iteration method.
The best known convergence result for policy and strategy iteration are given by \cite{ye2005new} and \cite{hansen2013strategy}. The best known iteration complexities for both problems are of the order $(1-\gamma)^{-1}$, which becomes unbounded as $\gamma\to 1$. It is worth mentioning that \cite{ye2005new} designed a combinatorial interior-point algorithm (CIPA) that solves the DMDP in strongly polynomial time. 

Sample-based algorithms for learning value and policy functions for MDP have been studied in \cite{kearns1999finite, kakade2003sample, singh1994upper, azar2011speedy, azar2013minimax, sidford2018variance, sidford2018near, agarwal2019optimality} and many others.
Among these papers, \cite{azar2013minimax} obtains the first tight sample bound for finding an $\epsilon$-optimal value function and for finding $\epsilon$-optimal policies in a restricted $\epsilon$ regime and
\cite{sidford2018near} obtains the first tight sample bound for finding an $\epsilon$-optimal \emph{policy} for any $\epsilon$.
Both sample complexities are of the form  $\wt{O}[|\cS||\cA|(1-\gamma)^{-3}]$.
Lower bounds have been shown in \cite{azar2011reinforcement, even2006action} and \cite{ azar2013minimax}. 
\cite{azar2013minimax} give the first tight lower bound $\Omega[|\cS||\cA|(1-\gamma)^{-3}]$. For undiscounted average-reward MDP, a primal-dual based method was proposed in \cite{wang2017randomized} which achieves sample complexity $\wt{O}(|\cS||\cA| t_{\mathrm{mix}}^2 c_{\max}^2/ c_{\min}^2)$, where $t_{\mathrm{mix}}$ is the worst-case mixing time and $c_{\max} / c_{\min}$ is the ergodicity ratio. 
Sampling-based method for two-player stochastic game has been considered in \cite{wei2017online} in an online learning setting. 
However, their algorithm leads to a sub-optimal sample-complexity when generalized to the generative model setting.

 As for general stochastic games, the minimax Q-learning algorithm and the friend-and-foe Q-learning algorithm were introduced in \cite{littman1994markov} and \cite{littman2001friend}, respectively. The Nash Q-learning algorithm was proposed for zero-sum games in \cite{hu2003nash} and for general-sum games in \cite{littman2001value, hu1999multiagent}. 
\section{Technique Overview}
Since  stochastic games are a generalization of MDPs, many techniques for solving MDPs can be immediately generalized to stochastic games.
However, as we have discussed, 
some of the techniques used to achieve optimal sample complexities for solving MDPs in a generative model do not have a clear generalization to stochastic games. 
Nevertheless, we show how to design an algorithm that carefully extends particular Q-learning based methods, i.e. methods that always maintain an estimator for the optimal value function (or $\bQ^*$), to achieve our goals. 

\paragraph{Q-Learning: }%
 To motivate our approach we first briefly review previous Q-learning based methods and the core technique that achieves near-optimal sample complexity. 
To motivate Q-learning, we first recall the value iteration algorithm solving an MDP. Given a full model for the MDP value iteration updates the iterates as follows
\[
\bv^{(i)}\gets \cT[\bv^{(i-1)}]:= V[\bQ(\bv^{(i-1)})]
\]
where $\bv^{(0)}$ can be an arbitrary vector. 
Since the Bellman operator is contractive and $\bv^*$ is a fix point of $\cT$, this method gives an $\epsilon$-optimal value in $O[(1-\gamma)^{-1}\log(\epsilon^{-1})]$ iterations.
In the learning setting, $\cT$ cannot be exactly computed.
The Q-learning approach estimates $\cT$ by its approximate version, i.e., to compute $\bP(\cdot~|~s,a)^\top\bv^{(i-1)}$, we obtain samples from $\bP(\cdot~|~s,a)$, and then compute the empirical average. Then we compute the approximate Q-value at the $i$-th iteration as
\[
\bQ^{(i)}=\wh{\bQ}[\bv^{(i-1)}] := \br + \wh{\bP} \bv^{(i-1)} \quad\text{and} \quad 
\wh{\cT}(\bv^{(i-1)}) := V[\wh{\bQ}(\bv^{(i-1)})],
\]
where
\[
\wh{\bP}(\cdot~|~s,a)^\top \bv = \frac{1}{m}\sum_{s_i\sim P(\cdot|s,a),~ i\in [m]} \bv(s_i)
\]
for some $m > 0$.
Then the estimation error per step is defined as
\[
\bepsilon^{(i)} = {\bQ}[\bv^{(i-1)}] - \wh{\bQ}[\bv^{(i-1)}] . 
\]
Since the exact value iteration takes at least $\Omega[(1-\gamma)^{-1}]$ iterations to converge, 
the Q-learning (or approximated value iteration) takes at least $\Omega[(1-\gamma)^{-1}]$ iterations.
The total number of samples used over all the iterations is the sample complexity of the algorithm.

\paragraph{Variance Control and Monotonicity Techniques:}
To obtain the optimal sample complexity for one-player MDP, 
one approach is to carefully bound each entry of $\bepsilon^{(i)}$. 
By Bernstein inequality (\cite{azar2013minimax,sidford2018near, agarwal2019optimality}), 
we have, with high probability,
\[
|\bepsilon^{(i)}|\lesssim\sqrt{\var(\bv^{(i-1)})/m}
\le \sqrt{\var(\bv^{*})/m} + \text{ lower-order terms}.
\]
where $\var(\bv) = \bP \bv ^2 -(\bP \bv)^2$ is the \emph{variance-of-value} vector and ``$\lesssim $" means ``approximately less than."
Let $\pi^{(i)}$ be a policy maintained in the $i$-th iteration (e.g.  the greedy policy of the current Q-value).
Due to the estimation error $\bepsilon^{(i)}$, the per step error bound reads, 
\begin{align*}
	\bQ^* - \bQ^{(i)}
	\lesssim \gamma \bP^{\pi^*}\bQ^* - \gamma\bP^{\pi^{(i-1)}} \bQ^{(i-1)} + \bepsilon^{(i)}.
\end{align*}
To derive the overall error accumulation, 
\cite{sidford2018near} use the crucial  \emph{monotonicity} property, i.e., since $\pi^{(i-1)}(s)=\arg\max_{a} \bQ^{(i-1)}(s,a)$, we have
\begin{align}
	\label{eqn:monotone}
	\bQ^{(i-1)}(s,\pi^*(s)) \le \bQ^{(i)}(s,\pi^{(i-1)}(s)).
\end{align}
We thus have
\[
\bQ^* - \bQ^{(i)}
\lesssim \gamma\bP^{\pi^*}\bQ^* - \gamma\bP^{\pi^*} \bQ^{(i-1)} + \bepsilon^{(i)}.
\]
By induction, 
we have
\begin{align}
	\label{eqn:induction-mdp}
	\bQ^* - \bQ^{(i)} \le (I-\gamma\bP^{\pi^*})^{-1} \sqrt{\var(\bv^{*})/m} + \text{lower-order terms}.
\end{align}
The leading-order error accumulation term $(I-\gamma\bP^{\pi^*})^{-1} \sqrt{\var(\bv^{*})/m}$
satisfies the so-called \emph{total variance} property, and can be upper bounded uniformly by 
$\sqrt{(1-\gamma)^{-3}m^{-1}}$, resulting the correct dependence on $(1-\gamma)$.
Therefore the monotonicity property allows us to use $\pi^*$ as a \emph{proxy} policy, which carefully bounds the error accumulation.
For the additional subtlety of how to obtain an optimal policy, please refer to \cite{sidford2018near} for the variance reduction technique and the monotone-policy technique.

Similar observations regarding MDPs was used in \cite{agarwal2019optimality} as well.
This powerful technique, however, does not generalize to the game case due to the \emph{lack of monotonicity}.
Indeed, \eqref{eqn:monotone} does not hold for stochastic games due to the existence of both minimization and maximization operations in the Bellman operator. This is the critical issue which this paper seeks to overcome.

\paragraph{Finding Monotone Value-Strategy Sequences for Stochastic Games:} 
Analogously to the MDP case, one approach is to bound error accumulation for stochastic games is to bound each entry of the error vector $\bepsilon^{(i)}$ carefully.
In fact, our method for solving stochastic games is very much like the MDP method used in \cite{sidford2018near}.
However, the analysis is much different in order to resolve the difficulty introduced by the lack of monotonicity.

Since a stochastic game has two players, we modify the variance reduced Q-value iteration (vQVI)  method in \cite{sidford2018near} to obtain a min-player strategy and a max-player strategy respectively.
Since the two players are symmetric, let us focus on introducing and analyzing the algorithm for the min-player.
By a slight modification of the vQVI method, we can guarantee to obtain a sequence of strategies and values, $\{\bv^{(i)}, \bQ^{(i)}, \sigma^{(i)}, \bepsilon^{(i)}\}_{i=0}^{R}$, that satisfy, with high probability, 
\begin{align}
	\label{eqn:informal-mdvss}
\begin{aligned}
&\text{1.}\quad \bv^{(0)}\ge \bv^{(1)} \ge \ldots  \bv^{(R)} \ge \bv^*;\\
&\text{2.}\quad \cT_{\sigma^{(i)}}[ \bv^{(i)}] \le \bv^{(i)}, \cT [\bv^{(i)}]\le \bv^{(i)}, \cH_{\pi_{\min}^{(i)}} [\bv^{(i)}] \le \bv^{(i)};
\end{aligned}
\qquad
\begin{aligned}
&\text{3.}\quad \bQ^{(i)} \le \bQ[\bv^{(i-1)}] + \bepsilon^{(i)};\\
&\text{4.} \quad
\bv^{(i)}\le V[\bQ^{(i)}].
\end{aligned}
\end{align}
where $\sigma^{(i)}=(\pi_{\max}^{(i)},\pi_{\min}^{(i)})$.
The first property guarantees  that the value sequences are monotonically decreasing, the second property guarantees $\bv^{(i)}$ is always an upper bound of the value $\bv^{\pi_{\min}^{(i)}}$, and the third and fourth inequality guarantees that  $\bv^{(i)}$ is well approximated by $V[\bQ^{(i)}]$ and the estimation error satisfy
\[
|\bepsilon^{(i)}|\lesssim\sqrt{\var(\bv^{(i)})/m},
\]
where $m$ is the total number of samples used per state-action pair. 
Note that, as long as we can guarantee that $\bv^{(R)} - \bv^* \le \epsilon$, we can guarantee the min-strategy $\pi^{(R)}_{\min}$ is also good:
\[
\bv^{*}\le \bv^{\pi^{(R)}_{\min}} \le \bv^{(R)}.
\]

\paragraph{Controlling Error Accumulation using Auxiliary Markovian Strategy:}
Due to the lack of monotonicity \eqref{eqn:monotone}, we cannot use the optimal strategy $\sigma^*$ as a proxy strategy to carefully account for the error accumulation.
To resolve this issue, we construct a new proxy strategy $\ms$. This strategy is a Markovian strategy, which is time-dependent but not history dependent, i.e., at time $t$, the strategy played is a deterministic map $\ms_{t}:\cS\rightarrow\cA$.
The proxy strategy satisfies the following: 
\begin{enumerate}
	\item \textbf{Underestimation.} its value, $\bv[\ms_{i}]$, (expected discounted cumulative reward starting from any time) is upper bounded by $\bv^*$;
	\item \textbf{Contraction.} $\bv^{(i)}(s) - \bv[\ms_{i}](s)
	\le \gamma \bP(\cdot|s,\ms_{i}(s))^{\top}\big(\bv^{(i-1)}- \bv[\ms_{i-1}]\big)  + \bepsilon^{(i)}(s,\ms_{i}(s))$, 
\end{enumerate}
Similarly, we can bound the error $\bepsilon^{(i)}(s,\ms_{i}(s))$ by the variance-of-value of the proxy strategy
\[
\bepsilon^{(i)}(s,\ms_{i}(s))
\le \sqrt{\var(\bv[\ms_{i}])(s,\ms_{i}(s))/m} + \text{ lower-order terms}.
\]
Based on the first property, we can upper bound
\[
\bv^{(i)} - \bv^{*} \le \bv^{(i)} - \bv[\ms_{i}].
\]
Based on the second property, and induction on $i$, we can now write a new form of error accumulation,
\[
\bv^{(R)} - \bv^*\lesssim \sum_{i=1}^{R}\gamma^{R-i}  
\bP^{\ms_{R}} \cdot \bP^{\ms_{R-1}} \cdot \ldots \cdot \bP^{\ms_{i+1}}\sqrt{\var(\bv[\ms_{i-1}])_{\ms_{i}}/m}+ \text{ lower-order terms},
\]
where $\var(\bv[\ms_{i-1}])_{\ms_{i}}(s) :=\var(\bv[\ms_{i}])(s,\ms_{i}(s))$ for all $s\in \cS$.
We derive 
a new \emph{law of total variance} bound for the first term and ultimately prove an error accumulation upper bound:
\[
\bv^{(R)} - \bv^*\lesssim \sqrt{(1-\gamma)^{-3}m}+ \text{ lower-order terms},
\]
giving the optimal sample bound.

\section{Sample Complexity of Stochastic Games}
In this section, we provide and analyze our sampling-based algorithm for solving stochastic games. Recall that we have a \emph{generative model} for the game such that we can obtain 
samples from state-action pairs.
 Each sample is obtained in time $O(1)$. As such we care about the total number of samples used or the total amount of time consumed by the algorithm.
We will provide an efficient algorithm that takes input a generative model and obtains a good strategy for the underlying stochastic game.

We now describe the algorithm.
Since the min-player and max-player are symmetric, let us focus on the min-player strategy. 
For the max player strategy, we can either consider the game $\cG'=(\cS_{\min}, \cS_{\max}, \bP, \one-\br, \gamma)$, in which the roles of the max and min players switched, or use the corresponding algorithm for the max-player defined in Section~\ref{sec:alg}, an algorithm that is a direct generalization from the min-player algorithm.

\begin{algorithm}[htb!]
	\caption{
		QVI-MDVSS:
		\label{alg-halfErr} 
		algorithm for computing monotone decreasing value-strategy sequences.
	}
	{\small
		\begin{algorithmic}[1]
			\State 
			\textbf{Input:} A generative model for stochastic game, $\cM=(\cS, \cA, \br, \bP, \gamma)$;
			\State
			\textbf{Input:} Precision parameter $u\in[0,(1-\gamma)^{-1}]$,
			and error probability $\delta \in (0, 1)$;
			\State 
			\textbf{Input:}
			Initial values $\bv^{+(0)}, \sigma^{+(0)}$ that satisfies monotonicity:
			{\small
				\vspace{-2mm} 
				\begin{align}
					\label{eqn:input-condition}
					\bv^*\le \bv^{+(0)} \le \bv^* +u\one,\quad \bv^{+(0)}\ge \cT [\bv^{+(0)}],\quad\text{and} \quad \bv^{+(0)}\ge \cT_{\sigma^{+(0)}} [\bv^{+(0)}];
				\end{align}
				\vspace{-4mm}
			}
			\State\textbf{Output:}
			$\{\bv^{+(i)}, \bQ^{+(i)}, \sigma^{+(i)}, \bxi^{+(i)} \}_{i=0}^{R}$ which is an MDVSS with probability at least $1-\delta$;
			\State
			\State\textbf{INITIALIZATION:} 
			\State Let $c_1, c_2, c_3, c$ be some tunable absolute constants;
			\State \textcolor{OliveGreen}{\emph{\textbackslash \textbackslash Initialize constants:}}
			\State \qquad$\beta\gets (1-\gamma)^{-1}$, and $R\gets\lceil c_1\beta\ln[\beta u^{-1}]\rceil$; 
			\State \label{def:m1}\qquad$m_1 \gets{c_2\beta^3\cdot\min(1,u^{-2})\cdot{\log(8|\cS||\cA|\delta^{-1})} }{}$; 
			\State  \qquad$m_2\gets {c_3\beta^{2}\log[2R|\cS||\cA|\delta^{-1}]}$;
			\State\qquad $\alpha_1\gets L/m_1$ where $L = c\log(|\cS||\cA|\delta^{-1}(1-\gamma)^{-1}u^{-1})$; 
			\State
			\textcolor{OliveGreen}{\emph{\textbackslash \textbackslash Obtain an initial batch of samples:}}
			\State For each  $(s, a)\in \cS\times\cA$:
			obtain independent samples $s_{s,a}^{(1)}, s_{s,a}^{(2)}, \ldots, s_{s,a}^{(m_1)}$ from $\bP(\cdot | s,a)$;
			\State
			Initialize: 
			$\bw^{+}=\wt{\bw}^+ = \wh{\bsigma}^+=\bQ^{+(0)}=\bQ^{+(1)} \gets \beta \cdot {\bf 1}_{\cS \times \cA}$ and $i\gets 0$;	
			\For{each $(s, a)\in \cS\times\cA$} 
			\State\label{def:start-init-compute} \textcolor{OliveGreen}{\emph{\textbackslash \textbackslash Compute empirical estimates of $\bP_{s,a}^{\top}\bv^{+(0)}$ and $\var({\bv^{+(0)}})(s,a)$:}}
			\State 		\label{alg1: compute w1+} 
			$\wt{\bw}^{+}(s,a) \gets \frac{1}{m_1} 
			\sum_{j=1}^{m_1} \bv^{+ (0)}(s_{s,a}^{(j)})$;
			
			\State  
			\label{def:empricial-variance}
			$\wh{\bsigma}^{+}(s,a)\gets \frac{1}{m_1} \sum_{j=1}^{m_1}(\bv^{+ (0)})^2(s_{s,a}^{(j)}) - (\wt{\bw}^{+})^2(s,a)$ ;
			
			\State \textcolor{OliveGreen}{\emph{\textbackslash \textbackslash Shift the empirical estimate to have one-sided error and guarantee monotonicity:}} 
			\State
			$\bw^+(s, a) \gets \wt{\bw}^+(s,a) + \sqrt{\alpha_1\wh{\bsigma}^+(s,a)} + \alpha_1^{3/4}\beta$
			\State \textcolor{OliveGreen}{\emph{\textbackslash \textbackslash Compute coarse estimate of the  $Q$-function and make sure its value is in $[0,\beta]$:}}
			\State
			$\bQ^{+(0)}(s,a) \gets \min[\br(s,a) + \gamma \bw^{+}(s,a), \beta]$\label{def:end-init-compute} 

			\EndFor
			\State
			\State\textbf{REPEAT:} \textcolor{OliveGreen}{\emph{\qquad\qquad\textbackslash \textbackslash successively improve}}
			\For{$i=1$ to $R$}
			\State  \textcolor{OliveGreen}{\emph{\textbackslash \textbackslash Compute
					the one-step dynamic programming:}}
			\State\label{alg: pit} Let ${\bv}^{+(i)} \gets \wt{\bv}^{+(i)}\gets \cT [ \bQ^{+(i-1)} ]$, ${\sigma}^{+(i)}\gets\wt{\sigma}^{+(i)}\gets \sigma(\bQ^{+(i-1)})$;
			\State \textcolor{OliveGreen}{\emph{\textbackslash \textbackslash Compute strategy and value and maintain monotonicity:}} 
			\State \label{alg: v2+} For {each $s\in \cS$} if ${\bv}^{+(i)}(s)\ge \bv^{+(i-1)}(s)$, then 
			$\bv^{+(i)}(s)\gets \bv^{+(i-1)}(s)$ and $\sigma^{+(i)}(s)\gets \sigma^{+(i-1)}(s)$;
			
			\State\textcolor{OliveGreen}{\emph{\textbackslash \textbackslash Obtaining a small batch of samples:}}
			\State For each $(s, a)\in \cS\times\cA$:
			draw independent samples $\wt{s}_{s,a}^{(1)}, \wt{s}_{s,a}^{(2)}, \ldots, \wt{s}_{s,a}^{(m_2)}$ from $\bP(\cdot | {s,a})$;
			\State \textcolor{OliveGreen}{\emph{\textbackslash \textbackslash Compute the expected value,  $\bg^{\pm(i)}$, the estimate of $\bP \big[\bv^{\pm(i)}  - \bv^{\pm (0)}\big]$ with one-sided error:}}
			\State  \label{alg1: compute g} Let $\wt{\bg}^{+(i)}(s,a)\gets {\frac{1}{m_2}} \sum_{j=1}^{m_2} \big[\bv^{+ (i)}(\wt{s}_{s,a}^{(j)}) - \bv^{+(0)}(\wt{s}_{s,a}^{(j)}) \big]$;
			\State Let ${\bg}^{+(i)}(s,a)\gets\wt{\bg}^{+(i)}(s,a)+ C(1-\gamma)u$, where $C>0$ is an absolute constant;
			\State \textcolor{OliveGreen}{\emph{\textbackslash \textbackslash Estimate the approximation error:}}
			\State $\bxi^{+ (i)}\gets 2\sqrt{\alpha_1 \bsigma_{\bv^{+ (0)}}} +
			2[\alpha_1^{3/4}\beta + C(1-\gamma)u]\cdot \one$
			\State  \textcolor{OliveGreen}{\emph{\textbackslash \textbackslash Improve $\bQ^{+ (i)}$  and make sure its value is in $[0,\beta]$:}}
			\State \label{alg: q-func} $\bQ^{+ (i+1)}\gets \min\Big[\br + \gamma\cdot[\bw^{+}+\bg^{+ (i)}], \beta\Big]$;
			\EndFor
			\State \textbf{return} 
			$\{\bv^{+(i)}, \bQ^{+(i)}, \sigma^{+(i)}, \bxi^{+ (i)}\}_{i=0}^{R}$
		\end{algorithmic}
	}
\end{algorithm}
\paragraph{The Full Algorithm.}
For simplicity, let us denote $\beta = 1/(1-\gamma)$.
Our full algorithm will use the QVI-MDVSS algorithm (Algorithm~\ref{alg-halfErr}) as a subroutine.
As we will show shortly, this subroutine maintains a monotonic value strategy sequence with high probability.
Suppose the algorithm is 
specified by an accuracy parameter $\epsilon\in(0,1]$.
We initialize a value vector $\bv^{(0)} = \beta\one$, and an arbitrary strategy $\sigma^{(0)}=(\pi_{\min}^{(0)}, \pi_{\max}^{(0)})$.
Let $u^{(0)} = \beta$.
Then our initial value and strategy satisfy the  requirement of the input specified by Algorithm~\ref{alg-halfErr}:
\begin{align*}
	\bv^*\le \bv^{(0)} \le \bv^* +u^{(0)}\one,\quad \bv^{(0)}\ge \cT [\bv^{(0)}],\quad\text{and} \quad \bv^{(0)}\ge \cT_{\sigma^{(0)}} [\bv^{(0)}];
\end{align*}
Let $u^{(j)}\gets \beta / 2^j$ and $\delta\gets 1/\poly(\log(\beta/\epsilon))$.
We  run Algorithm \ref{alg-halfErr} repeatedly:\\
\fbox{
\centering	
\parbox{0.90\textwidth}{
\begin{align}
	\label{algorithm-full}
(v^{(j+1)}, \sigma^{(j+1)})\gets\textrm{QVI-MDVSS}\gets (v^{(j)}, \sigma^{(j)}, u^{(j)}, \delta),
\end{align}
}}\\
where $\sigma^{(j)}=(\pi_{\min}^{(j)}, \pi_{\max}^{(j)})$ and we take the terminal value and strategy of the output sequence of Algorithm~\ref{alg-halfErr} as the input for the next iteration.
In total we run \eqref{algorithm-full} $R'=\Theta(\log(\beta/\epsilon))$ iterations.
In the end, we output $\pi^{(R')}_{\min}$ from $\sigma^{(R')}=(\pi_{\min}^{(R')}, \pi_{\max}^{(R')} )$ as our min-player strategy.

The formal guarantee of the algorithm is presented in the following theorem.
\begin{theorem}[Restatement of Theorem~\ref{mainthm}]
	\label{mainthm2}
	Given a stochastic game $\cG =(\cS_{\min}, \cS_{\max}, \bP, \br, \gamma)$ with a generative model, there exists (constructively) an algorithm  that outputs, with probability at least $1-\delta$,
	an $\epsilon$-optimal strategy $\sigma$ by querying 
	$Z=\wt{O}(|\cS| |\cA| (1-\gamma)^{-3} \epsilon^{-2})$
	samples in time $O(Z)$ using space $O(|\cS| |\cA|)$ where $\epsilon\in(0,1)$ and $\wt{O}(\cdot)$ hides $\poly\log[ |\cS||\cA|/(1-\gamma)/\epsilon/\delta]$ factors.
\end{theorem}

The formal proof of Theorem \ref{mainthm2} is given in the next section. 
Here we give a sketch of the proof.

\paragraph{Proof Sketch of Theorem \ref{mainthm2}:} 
We first show the high-level idea. Considering one iteration of \eqref{algorithm-full},
we claim that if the input value and strategy $\sigma^{(j)}, \bv^{(j)}, u^{(j)}$ satisfies the input condition \eqref{eqn:input-condition}, then  with probability at least $1-\delta$, the terminal value and strategy of the output sequence,
$\sigma^{(j+1)}, \bv^{(j+1)}$, satisfies,
\begin{align}
\label{eqn:half-err}
\bv^{\pi_{\min}^{j+1}}\le  \bv^{j+1}\le \bv^* + u^{(j)} \one/2 =: \bv^* + u^{(j+1)} \one;
\end{align}
and  $(\sigma^{(j+1)}, \bv^{(j+1)}, u^{(j+1)})$ satisfies the the input condition \eqref{eqn:input-condition}.
Namely, with high probability, the error of the output is decreased by at least half and the output can be used as an input to the QVI-MDVSS algorithm again.
Suppose we run the subroutine of Algorithm~\ref{alg-halfErr} for $R'$ times, and conditioning on the event that all the instances of QVI-MDVSS succeed, the final error of $\pi_{\min}^{(R')}$ is then at most $u^{(R')} = 2^{-R'}\beta = \epsilon$, as desired.
By setting $\delta = \delta'/R'$ for some $\delta'>0$, we have that all QVI-MDVSS instances succeed with probability at least $1-\delta'$.
It remains to show that the algorithm QVI-MDVSS works as claimed.

\emph{High-level Structure of Algorithm~\ref{alg-halfErr}.}
To outline the proof, 
we denote a \emph{monotone decreasing value-strategy sequence} (MDVSS) as $\{\bv^{(i)}, \bQ^{(i)}, \sigma^{(i)}, \bepsilon^{(i)}\}_{i=0}^{R}$,  satisfying \eqref{eqn:informal-mdvss}, where $\bv^{(i)},\bepsilon^{(i)} \in \RR^\cS, \bQ^{(i)}\in \RR^{\cS\times \cA}$ and $\sigma^{(i)} = (\pi_{\min}^{(i)}, \pi_{\max}^{(i)})\in \cA^\cS$.
A more formal treatment of the sequence is presented in Section~\ref{sec:monotone_sequence}.

We next introduce the high-level idea of Algorithm~\ref{alg-halfErr}.
The basic step of the algorithm is to do approximate value-iteration while preserving all monotonic properties required by an MDVSS, i.e.,
we would like to approximate 
\[
\bQ^{(i)}={\bQ}[\bv^{(i-1)}] := \br + {\bP} \bv^{(i-1)} \quad\text{and} \quad 
{\cT}[\bv^{(i-1)}] := V[{\bQ}(\bv^{(i-1)})].
\]
We would like to approximate ${\bP} \bv^{(i-1)}$ using samples,
but we do not want to use the same amount of samples per iteration (as it become costly if the number of iterations is large). 
Instead, we compute only the \emph{first} iteration (i.e., estimate $\bP \bv^{(0)}$) up to high accuracy with a large number of samples ($m_1$ samples, defined in Line~\ref{def:m1}).  
These computations are presented in Line~\ref{def:start-init-compute}-\ref{def:end-init-compute}.
To maintain an upper bound of the of the estimation error, we also compute the empirical variances of the updates in Line~\ref{def:empricial-variance}.
We shift upwards our estimates by the estimation error upper bounds to make our estimators one-sided, which is crucial to maintain the MDVSS properties.
For the subsequent steps (Line~\ref{alg: pit} - \ref{alg: q-func}), we use $m_2$ samples per iteration ($m_2\ll m_1$) to  
estimate $\bP(\bv^{(i)} - \bv^{(0)})$.
The expectation is that $(\bv^{(i)} - \bv^{(0)})$ has a small $\ell_\infty$ norm, and hence $\bP(\bv^{(i)} - \bv^{(0)})$ can be estimated up to high accuracy with only a small number of samples. 
The estimator of $\bP(\bv^{(i)} - \bv^{(0)})$ plus the estimator of $\bP \bv^{(0)}$ in the initialization steps gives a high-accuracy estimator (Line~\ref{alg: q-func}) for the value iteration.
Since $m_2\ll m_1$, the total number of samples per state-action pair is dominated by $m_1$.
This idea is formally known as \emph{variance-reduction}, firstly proposed for solving MDP in \cite{sidford2018variance}.
Similarly, we shift our estimators to be one-sided.
We additionally maintain carefully-designed strategies in Line~\ref{alg: pit}-\ref{alg: v2+} to preserve monotonicity. 
Hence the algorithm can be viewed as a value-strategy iteration algorithm.


\emph{Correctness of Algorithm~\ref{alg-halfErr}.}
We now sketch the proof of  correctness for Algorithm~\ref{alg-halfErr}.
Firstly Proposition \eqref{prop:main0} shows that the if an MDVSS, e.g., $\{\bv^{+(i)},\bQ^{+(i)}, \sigma^{+(i)}, \bepsilon^{+(i)}\}_{i=0}^R$,
	satisfies  $\|\bv^{+(R)}- \bv^*\|_{\infty} \le \epsilon$ for some $\epsilon>0$ then
	 their terminal strategies and values  satisfy
	\[
	  \bv^{\pi_{\min}^{+(R)}}\le  \bv^{+(R)}\le \bv^* + \epsilon \one.
	\]
	This indicates that as long as we can show $\epsilon \le u/2$, then the \emph{halving-error-property} \eqref{eqn:half-err} holds.
	
%
	Proposition~\ref{mainprop0} shows the halving-error-property can be achieved by setting 
	\[
	\bepsilon^{+(i)}
	\lesssim \sqrt{\var(\bv^{+(0)})/m} + \text{ lower-order terms},
	\]
	where $\var(\bv^{+(0)})$ is the variance-of-value vector of $\bv^{+(0)}$ and $m \gtrsim \sqrt{\beta^3 u^{-2}}$.
	This proof is based on constructing an auxiliary Markovian strategy for analyzing the error accumulation throughout the value-strategy iterations. The Markovian strategy is a time-dependent strategy used as a proxy for analyzing the entrywise error recursion (Lemmas \ref{mainprop0}-\ref{cor:error accur}). 

Proposition~\ref{mainprop} shows, with high probability, Algorithm~\ref{alg-halfErr} produces value-strategy sequences
$\{\bv^{+(i)}, \bQ^{+(i)}, \sigma^{+(i)}, \bxi^{+(i)} \}_{i=0}^{R}$, which is indeed an MDVSS and $\bxi^{+(i)}$ satisfies Proposition~\ref{mainprop0}. The proof involves analyzing the probability of ``good events" on which monotonicity is preserved at every iteration by using confidence estimates computed during the iterations and concentration arguments. See Lemmas \ref{def:events}-\ref{lemma:all good events happend prob} for the full proof of Proposition \ref{mainprop}. 

\emph{Putting Everything Together.} Finally by putting together the strategies, we conclude that the terminal strategy of the iteration~\eqref{algorithm-full} is always an approximately optimal min-player strategy to the game, with high probability.
For implementation, since our algorithm only computes the inner product based on samples, the total computation time is proportional to the number of samples.
Moreover, since we can update as samples are drawn and output the monotone sequences as they are generated, we do not need to store samples or the value-strategy sequences, thus the overall space is $O(|\cS||\cA|)$.
\qed

\section{Proof of Main Results}
The remainder of this section is devoted to proving Theorem~\ref{mainthm}. We prove this by formally providing a notion of \emph{monotone value-strategy sequences}.
With this, we show if an algorithm outputs some monotone value-strategy sequence, then the terminal strategy of the sequence is always an approximately optimal strategy to the game.
We then show that Algorithm~\ref{alg-halfErr}
produces {monotone value-strategy sequences} with high probability.


\subsection{Additional Notation}

First we  provide additional notation critical to our proofs.

\paragraph{Markovian Strategies:} 
We denote a Markovian strategy $\ms$ as an infinitely long sequence of pre-defined strategies \[\ms:= (\sigma_1, \sigma_2, \ldots),
\] where each $\sigma_i$ is a normal deterministic strategy.
We denote 
\[
\ms_{t} = (\sigma_t, \sigma_{t+1}, \ldots)
\]
as another Markovian strategy.
We denote $\ms_{\min}$ and $\ms_{\max}$ as the min-player strategy and the max-player strategy respectively.
When using the strategy, players uses $\sigma_t$ at time $t$.
The strategy is Markovian because it does not depend on the historical moves. 
Note that a stationary strategy $\sigma$ is a special case of the Markovian strategy: $\sigma = (\sigma, \sigma, \sigma,\ldots )$. 
The value of a Markovian strategy is defined as before, but the states are generated by playing the action $\sigma_t(s^t)$ at time $t$.
Since the strategy has a time dependence, we denote
\[
\bv_t^{\ms}:=\bv[\ms_{t}] \quad
\text{and}\quad
\bQ^{\ms}_t = \br + \gamma \bP \bv_{t+1}^{\ms}. 
\]
The (half) Bellman operators are defined similarly to that of stationary policies.

\subsection{Monotone Value-Strategy Sequence}
\label{sec:monotone_sequence} 
In this section we formally define monotone strategy value sequences.
Such a sequence, although not explicitly stated in \cite{sidford2018variance, sidford2018near}, are crucial for these algorithms to obtain good policy while obtaining a good value for an MDP.  
In the following sections, we denote 
 $m\ge 1$, $L\ge 1$ and $\epsilon\in[0, (1-\gamma)^{-1}]$ as parameters.
{Monotone value-strategy sequences} are formally defined as follows.
\begin{definition}[Monotone Decreasing Value-Strategy Sequence]
	\label{def:mdvps}
	A \emph{monotone decreasing value-strategy sequence} (MDVSS) is a sequence of $\{\bv^{(i)}, \bQ^{(i)}, \sigma^{(i)}, \bepsilon^{(i)}\}_{i=0}^{R}$ where $\bv^{(i)},\bepsilon^{(i)} \in \RR^\cS, \bQ^{(i)}\in \RR^{\cS\times \cA}$ and $\sigma^{(i)} = (\pi_{\min}^{(i)}, \pi_{\max}^{(i)})\in \cA^\cS$ satisfy 
	\begin{enumerate}
		\item $\bv^{(0)}\ge \bv^{(1)} \ge \ldots  \bv^{(R)} \ge \bv^*$;
		\item $\forall i\in[0,R]$, $\cT_{\sigma^{(i)}}[ \bv^{(i)}] \le \bv^{(i)}, \cT [\bv^{(i)}]\le \bv^{(i)}, \cH_{\pi^{(i)}_{\min}}[ \bv^{(i)}]\le \bv^{(i)}$; 
		\item $\forall i\in[R]$, $\bQ^{(i)} \le \br + \gamma \bP \bv^{(i-1)} + \bepsilon^{(i)}$; 
		\item $\forall i\in[R]$,  $\bv^{(i)}\le V[ \bQ^{(i)}]$. 
	\end{enumerate}
	Note that $\bQ^{(0)}, \bepsilon^{(0)}$ can be arbitrary.
\end{definition}
Here, we explain the intuition of the sequence.
The first property guarantees that the value-estimator $\bv^{(i)}$s always upper bound the optimal value.
The second property guarantees that
$\bv^{\pi_{\min}}\le \bv^{(i)}$.
Indeed 
\[
\bv^{\pi_{\min}} = \lim_{t\rightarrow \infty}
\cH_{\pi_{\min}}^{t}[ \bv^{(i)}] \le \bv^{(i)},
\]
where $\cH_{\pi_{\min}}^{t}$ denotes applying $\cH_{\pi_{\min}}$ for $t$ times. 
Therefore, as long as $\bv^{(R)} - \bv^* \le \epsilon\one$, we have
\[
\bv^*\le \bv^{\pi_{\min}} \le \bv^*  + \epsilon\one.
\]
The third and the fourth property guarantees $\bv^{(R)}$ is good by requiring that $\bv^{(i)}$ and $\bQ^{(i)}$ satisfy the approximate value iteration with one-sided error. However the overall error $\bv^{(R)} - \bv^* $ is controlled by the per-step error term $\bepsilon^{(i)}$.

Similarly, we define \emph{monotone increasing value-strategy sequence}(MIVSS) analagously with every inequality reversed.
\begin{definition}[Monotone Increasing Value-Strategy Sequence]
	\label{def:mivps}
	A \emph{monotone increasing value-strategy sequence} (MIVSS) is a sequence of $\{\bv^{(i)}, \bQ^{(i)}, \sigma^{(i)}, \bepsilon^{(i)}\}_{i=0}^{R}$ where $\bv^{(i)}\in \RR^\cS, \bQ^{(i)}\in \RR^{\cS\times \cA}$ and $\sigma^{(i)}\in \cA^\cS$ that satisfies,
	\begin{enumerate}
		\item $\bv^{(0)}\le \bv^{(1)} \le \ldots  \bv^{(R)} \le  \bv^*$;
		\item $\forall i\in[0,R]$,  $\cT_{\sigma^{(i)}} [\bv^{(i)}] \ge \bv^{(i)}, \cT \bv^{(i)}\ge [\bv^{(i)}], \cH_{\pi_{\max}^{(i)}}[\bv^{(i)}]\ge \bv^{(i)}$;
		\item $\forall i\in[R]$, $\bQ^{(i)} \ge \br + \gamma \bP \bv^{(i-1)} - \bepsilon^{(i)}$;
		\item $\forall i\in[R]$, $\bv^{(i)}\ge V [\bQ^{(i)}]$.
	\end{enumerate}
	Note that $\bQ^{(0)}, \bepsilon^{(0)}$ can be arbitrary.
\end{definition}

\subsection{Monotone Value-Strategy Sequence Implies Good Strategy}
Next, we show that MDVSS or MIVSS implies a good terminal value/strategy.
First we show that if the terminal value $\bv^{(R)}$ is close to the optimal value, then we are  guaranteed to have good strategies as well.

\begin{proposition}
	\label{prop:main0}
	Suppose we have an MDVSS, $\{\bv^{(i)},\bQ^{(i)}, \sigma^{(i)}, \bepsilon^{(i)}\}_{i=0}^R$,
	with $\|\bv^{(R)}- \bv^*\|_{\infty} \le \epsilon$ for some $\epsilon\ge 0$.
	Then we have
	\[
	\bv^{\pi_{\min}^{(R)}}\le \bv^* + \epsilon \one.
	\]
	Similarly, suppose $\{\bv^{(i)},\bQ^{(i)}, \sigma^{(i)}, \bepsilon^{(i)}\}_{i=0}^R$ is an MIVSS, then
	\[
	\bv^{\pi_{\max}^{(R)}}\ge \bv^* - \epsilon \one.
	\]
\end{proposition}
\begin{proof}
By the property of an MDVSS, we have\[
\bv^{\pi_{\min}^{(R)}} \le \bv^{(R)}.
\]
Since $\bv^{(R)} \le \bv^* + \epsilon\one$, we prove the first inequality.
The second inequality follows similarly.
\end{proof}

Next we consider when it is the case we achieve a good terminal value. The following proposition shows that an MDVSS(MVISS) with an appropriate error parameters has a better terminal value than its initial value.
\begin{proposition}
	\label{mainprop0}
	Let $u\in(0,\beta), \beta = (1-\gamma)^{-1}$, $R=\Theta[\beta\log(\beta/u)]$.
	Suppose 
	an MDVSS (or MIVSS)  $\{\bv^{(i)},\bQ^{(i)}, \sigma^{(i)}, \bepsilon^{(i)}\}_{i=0}^R$
	satisfies
	\[
	 \|\bv^{(0)} - \bv^*\|_{\infty}\le u\quad\text{and}\quad
	\bepsilon^{(i)} = \sqrt{L \cdot\var({\bv^{(0)}}) /m} + \beta\cdot (L/m)^{3/4}+  u/ (CR),
	\]
	for some large constant $C>1$ and $m\ge 1$.
	Then we have
	\[
	\|\bv^{(R)}- \bv^*\|_{\infty} \le u /2
	\quad\text{for}\quad
	m = \wt{\Omega}\bigg(\frac{1}{\min(1, u^2)\cdot (1-\gamma)^3}\bigg)
	~.
	\]
\end{proposition}
Note that Proposition~\ref{mainprop0} shows that in an MDVSS/MIVSS, the distance to the optimal value of the terminal value reduces by at least half of its initial value. 
Starting from some $\bv^{(0)}$ with distance at most $\beta$ to $\bv^*$, by concatenating $O(\log(\beta/\epsilon))$ many MDVSS/MIVSS's, with the initial value of one sequence set as the terminal value of the last sequence,  an $\epsilon$-optimal value can be obtained.
The remainder of this subsection devotes to proving the above proposition.
Since MIVSS and MDVSS are symmetric, in the following analysis, we focus on MDVSS and the analysis follows similarly for MIVSS.
\subsubsection{Auxiliary Markovian Strategy}
Due to the lack of monotonicity  we do not know how to use the optimal strategy $\sigma^*$ to carefully account for the error accumulation of the MDVSS.
To resolve this issue, we instead use the following auxiliary Markovian strategy as such a proxy.
\begin{definition}[Auxiliary Strategy]
	Given a MDVSS, $\{\bv^{(i)},\bQ^{(i)}, \sigma^{(i)}, \bepsilon^{(i)}\}_{i=0}^R$, we
	denote the Markovian auxiliary strategy for the max-player as \[
	\mpi{(i)}_{\aux\max}
	 = (\pi^{(i)}_{\aux\max}, \pi^{(i-1)}_{\aux\max}, \ldots, \pi^{(1)}_{\aux\max}, \pi^{*}_{\max},  \pi^{*}_{\max}, \pi^{*}_{\max} \ldots ),
	\]
	where $\pi^{(i)}_{\aux\max}(s) = \arg\max_{a}\bQ^{(i)}(s,a)$ for $s\in \cS_{\max}$.
	We denote the auxiliary strategy for the min-player as 
	\[
	\mpi{(i)}_{\aux\min} = \sigma_{\min}[\pi^{\infty(i)}_{\aux\max}]
	= (\pi^{(i)}_{\aux\min}, \pi^{(i-1)}_{\aux\min}, \ldots, \pi^{(1)}_{\aux\min}, \pi^{*}_{\min},  \pi^{*}_{\min}, \pi^{*}_{\min} \ldots ),
	\]
	which is the optimal counter Markovian policy of $\pi^{\infty(i)}_{\aux\max}$, i.e.,
	\[
	\forall s\in \cS_{\min}:\quad \pi^{\infty(i)}_{\aux\min}(s)
	= \arg \min_{a}\big[\br(s,a) + \gamma \bP(\cdot | s, a)^\top \bv[\pi^{\infty(i-1)}_{\aux}]\big].
	\]
	We also denote \[\sigma^{\infty(i)}_{\aux} 
	=\bigg[( \pi^{(i)}_{\aux\min}, \pi^{(i)}_{\aux\max}), ( \pi^{(i-1)}_{\aux\min},\pi^{(i-1)}_{\aux\max}), \ldots, ( \pi^{(1)}_{\aux\min},\pi^{(1)}_{\aux\max}), \sigma^{*},  \sigma^{*}, \sigma^{*} \ldots \bigg]
	.
	\]
	Furthermore, we denote 
	$
	\sigma^{(i)}_{\aux} = (\pi^{(i)}_{\aux\min}, \pi^{(i)}_{\aux\max})
	$ for $i\ge 1$ and $\sigma^{(i)}_{\aux} = \sigma^*$ for $i\le 0$.
\end{definition}

For a Markovian strategy, we first show that the strategy has a value always smaller than the optimal value.
\begin{lemma}
	For all $i\in [R]$, we have
	\[
	\bv\big[\sigma^{\infty(i)}_{\aux}\big]\le \bv^*.
	\]
\end{lemma}
\begin{proof}
	Denote
	\[
	\wt{\sigma}^{\infty(i)}_{\aux} =  \bigg[( \pi^*_{\min}, \pi^{(i)}_{\aux\max}), ( \pi^*_{\min},\pi^{(i-1)}_{\aux\max}), \ldots, ( \pi^*_{\min},\pi^{(1)}_{\aux\max}), \sigma^{*},  \sigma^{*}, \sigma^{*} \ldots \bigg].
	\]
	Denote $\sigma^{\infty(0)}_{\aux}= \wt{\sigma}^{\infty(0)}_{\aux} = (\sigma^*, \sigma^*, \ldots, )$.
	We first show that for all $i\in [R]$,
	$\bv\big[\sigma^{\infty(i)}_{\aux}\big]\le \bv\big[\wt{\sigma}^{\infty(i)}_{\aux}\big]$.
	Indeed it holds trivially for $i=0$. 
	Suppose it holds for some $i\ge 0$.
	Then, for each $s\in \cS_{\min}$, we have,
	\begin{align*}
	\bv[\sigma^{\infty(i)}_{\aux}](s)
	&= \min_{a}\big[\br(s,a) + \gamma \bP(\cdot ~|~ s, a)^\top \bv(\sigma^{\infty(i-1)}_{\aux})\big] \\
	&\le \big[\br(s,\sigma^*(s)) + \gamma \bP(\cdot ~|~ s, \sigma^*(s))^\top \bv(\sigma^{\infty(i-1)}_{\aux})\big]  \\
	&\le \big[\br(s,\sigma^*(s)) + \gamma \bP(\cdot ~|~ s, \sigma^*(s))^\top \bv(\wt{\sigma}^{\infty(i-1)}_{\aux})\big] \qquad(\text{due to } \bv\big[\sigma^{\infty(i-1)}_{\aux}\big]\le \bv\big[\wt{\sigma}^{\infty(i-1)}_{\aux}\big])\\
	& = \bv[\wt{\sigma}^{\infty(i)}_{\aux}](s).
	\end{align*}
	For each $s\in \cS_{\max}$, we have,
	\begin{align*}
	\bv[\sigma^{\infty(i)}_{\aux}](s)
	&= \big[\br(s,\sigma^{(i)}_{\aux}(s)) + \gamma \bP(\cdot | s, \sigma^{(i)}_{\aux}(s))^\top \bv(\sigma^{\infty(i-1)}_{\aux})\big] \\
	&\le\big[\br(s,\sigma^{(i)}_{\aux}(s)) + \gamma \bP(\cdot | s, \sigma^{(i)}_{\aux}(s))^\top \bv(
	\wt{\sigma}^{\infty(i-1)}_{\aux})\big] \qquad(\text{due to } \bv\big[\sigma^{\infty(i-1)}_{\aux}\big]\le \bv\big[\wt{\sigma}^{\infty(i-1)}_{\aux}\big])\\
	& = \bv[\wt{\sigma}^{\infty(i)}_{\aux}](s).
	\end{align*}
	Now, since $(\pi^*_{\max}, \pi^*_{\max}, \ldots)$ is the optimal counter strategy of $(\pi^*_{\min}, \pi^*_{\min}, \ldots)$, we have
	\[
	\bv[\wt{\sigma}^{\infty(i)}_{\aux}] \le \bv^* 
	\]
	holds similarly. This concludes the proof.
\end{proof}

Consider the error vector $\bepsilon^{(i)}$.
Recall that $ \bepsilon^{(i)}_{\sigma_{\aux}^{(i)}}$ denotes a vector in $\RR^{\cS}$ whose $s$-th entry is given by $\bepsilon^{(i)}(s, \sigma_{\aux}^{(i)}(s))$.
The next lemma shows a recursive relation between a Markovian strategy and the corresponding MDVSS values.
\begin{lemma}
	For all $i\in [R]$, we have 
	\[
	\bv^{(i)} - \bv[{\sigma^{\infty(i)}_{\aux}}] \le \gamma \bP^{\sigma_{\aux}^{(i)}}\big(\bv^{(i-1)} - \bv[{\sigma_{\aux}^{\infty(i-1)}}]\big) + \bepsilon^{(i)}_{\sigma_{\aux}^{(i)}}
	\]
\end{lemma}
\begin{proof}
	Note that $\bv^{(i)}\ge \bv^*\ge \bv[{\sigma^{\infty(i)}_{\aux}}]$.
	For each $s\in \cS_{\min}$, we have
	\begin{align*}
	\bv^{(i)}(s)
	&\le \min_{a} \bQ^{(i)}(s,a)
	\le \bQ^{(i)}(s,{\sigma^{(i)}_{\aux}}(s))\\
	&\le  \br(s,{\sigma^{(i)}_{\aux}}(s)) + \gamma \bP(\cdot|s,{\sigma^{(i)}_{\aux}}(s))^{\top} \bv^{(i-1)} + \bepsilon^{(i)}(s,{\sigma^{(i)}_{\aux}}(s)),
	\end{align*}
	and 
	\begin{align*} 
	\bv[{\sigma^{\infty(i)}_{\aux}}](s)  = \br(s,{\sigma^{(i)}_{\aux}}(s)) + \gamma \bP(\cdot|s,{\sigma^{(i)}_{\aux}}(s))^{\top} \bv[{\sigma^{\infty(i-1)}_{\aux}}].
	\end{align*}
	Thus
	\[
	\bv^{(i)}(s) - \bv[{\sigma^{\infty(i)}_{\aux}}](s)
	\le \gamma \bP(\cdot|s,{\sigma^{(i)}_{\aux}}(s))^{\top}\big(\bv^{(i-1)}- \bv[{\sigma^{\infty(i-1)}_{\aux}}]\big)  + \bepsilon^{(i)}(s,{\sigma^{(i)}_{\aux}}(s)).
	\]
	Similarly, for each $s\in \cS_{\max}$, we have,
	$\bv^{(i)}(s)\le \max_{a}\bQ^{(i)}(s,a) := \bQ^{(i)}(s,{\sigma^{(i)}_{\aux}}(s))$, thus
	\begin{align*}
	\bv^{(i)}(s) - \bv[{\sigma^{\infty(i)}_{\aux}}](s)
	&\le 
	\bQ^{(i)}(s,{\sigma^{(i)}_{\aux}}(s)) - \bv[{\sigma^{\infty(i)}_{\aux}}](s)\\
	&\le \gamma \bP(\cdot|s,{\sigma^{(i)}_{\aux}}(s))^{\top}\big(\bv^{(i-1)}- \bv[{\sigma^{\infty(i-1)}_{\aux}}]\big)  + \bepsilon^{(i)}(s,{\sigma^{(i)}_{\aux}}(s))
	\end{align*}
	as desired.
\end{proof}
With an inductive application of the above lemma, we obtain
the following corollary, which states an upper bound between the difference of $\bv^{(R)}$ and $\bv[{\sigma^{\infty(R)}_{\aux}}]$.
It connects the upper bound with a recursive propagation of the error.
\begin{corollary} 
	\label{corr:expansion}
	\begin{align*}
	\bv^{(R)} - \bv[{\sigma^{\infty(R)}_{\aux}}] 
	\le \gamma^{R}& \bP^{\sigma_{\aux}^{(R)}}\cdot \bP^{\sigma_{\aux}^{(R-1)}} \cdot \ldots \bP^{\sigma_{\aux}^{(1)}}\big(\bv^{(0)} - \bv^*\big)\\
	& + \sum_{i=1}^{R}\gamma^{R-i}  
	\bP^{\sigma_{\aux}^{(R)}} \cdot \bP^{\sigma_{\aux}^{(R-1)}} \cdot \ldots \bP^{\sigma_{\aux}^{(i+1)}}  \bepsilon^{(i)}_{\sigma_{\aux}^{(i)}}.
	\end{align*}
\end{corollary}
By this corollary, we know that the major error accumulation term is the second term.

\subsubsection{Error Accumulation}
We now consider the error accumulation in the sequence. 
As will show shortly, we relate $\bepsilon^{(i)}$ to the variance vector
$\sqrt{\var(\bv[\sigma_{\aux}^{\infty(i-1)}])}\in \RR^{\cS\times\cA}$, where
$\var(\bv)[s,a]
\defeq \var_{s'\sim P(\cdot|s,a)}[\bv(s')]$, $\forall (s,a), \bv$.
Therefore, it suffices to consider the following bound.

\begin{lemma}
	\begin{align*}
	\sum_{i=1}^{R}&\gamma^{R-i}  
	\bP^{\sigma_{\aux}^{(R)}} \cdot \bP^{\sigma_{\aux}^{(R-1)}} \cdot \ldots \bP^{\sigma_{\aux}^{(i+1)}}\sqrt{\var(\bv[\sigma_{\aux}^{\infty(i-1)}])_{\sigma_{\aux}^{(i)}}}\\
	&\le \sqrt{R\sum_{i=1}^{R}\gamma^{2(R-i)}  
		\bP^{\sigma_{\aux}^{(R)}} \cdot \bP^{\sigma_{\aux}^{(R-1)}} \cdot \ldots \bP^{\sigma_{\aux}^{(i+1)}}\var(\bv[\sigma_{\aux}^{\infty(i-1)}])_{\sigma_{\aux}^{(i)}}}
	\end{align*}
\end{lemma}
\begin{proof}
	Follows from Cauchy-Schwarz and that the $\bP$ matrices are non-negative with each row summing to $1$. 
\end{proof}

The following lemma establishes a Bellman-like equation for the variance vector of a Markovian strategy.
\begin{lemma}\label{lem:variance}
	For any Markovian strategy $\pi^{\infty} = (\pi^{(0)}, \pi^{(1)}, \ldots, )$, 
	we have, for all $s\in \cS$
	\begin{align}
	\label{eqn:newvar}
	\var\bigg[\sum_{t=0}^\infty \gamma^{t}\br(s^t, \pi^{(t)}(s))\bigg|s^0 = s\bigg] 
	= 
	\bigg[\sum_{t=0}^{\infty}\gamma^{2(t+1)}  
	\bP^{\pi^{(0)}}\bP^{\pi^{(1)}}\bP^{\pi^{(2)}}\ldots \bP^{\pi^{(t-1)}} \var[\bv(\pi^{\infty(t+1)})]_{\pi^{(t)}}\bigg](s)
	\end{align}
\end{lemma}
\begin{proof}
	\begin{align*}
	\var\bigg[\sum_{t=0}^\infty \gamma^{t}\br(s^t, \pi^{(t)}(s^t))\bigg|s^0 = s\bigg] 
	&= \EE\bigg[\bigg(\sum_{t=0}^\infty \gamma^{t}\br(s^t, \pi^{(t)}(s^t))\bigg)^2\bigg|s^0 = s\bigg]
	- \EE\bigg[\sum_{t=0}^\infty \gamma^{t}\br(s^t, \pi^{(t)}(s^t))\bigg|s^0 = s\bigg]^2.
	\end{align*}
	For the second term, we have 
	\begin{align*}
	\EE\bigg[\sum_{t=0}^\infty \gamma^{t}\br(s^t, \pi^{(t)}(s^t))\bigg|s^0 = s\bigg]^2
	&= \bv[\pi^{\infty(0)}]^2(s)= 
	\br(s, \pi^{(0)}(s))^2 \\
	&\qquad\qquad+ \gamma^2 (\bP^{\pi^{(0)}} \bv[\pi^{\infty(1)}])^2(s) + 2\gamma \br(s, \pi^{(0)}(s))(\bP^{\pi^{(0)}} \bv[\pi^{\infty(1)}])(s).
	\end{align*}
	For the first term, we have
	\begin{align*}
	\bigg(\sum_{t=0}^\infty \gamma^{t}\br(s^t, \pi^{(t)}(s^t))\bigg)^2
	=\br(s, \pi^{(0)}(s))^2 
	+ 2\br(s, \pi^{(0)}(s))\bigg(\sum_{t=1}^\infty \gamma^{t}\br(s^t, \pi^{(t)}(s^t))\bigg)
	+\bigg(\sum_{t=1}^\infty \gamma^{t}\br(s^t, \pi^{(t)}(s^t))\bigg)^2
	\end{align*}
	Note that
	\begin{align*}
		\EE&\bigg[\bigg(\sum_{t=1}^\infty \gamma^{t}\br(s^t, \pi^{(t)}(s^t))\bigg)^2\bigg| s^0=s\bigg]\\
		&= 
		\EE\bigg[\bigg(\sum_{t=1}^\infty \gamma^{t}\br(s^t, \pi^{(t)}(s^t))\bigg)^2\bigg| s^0=s\bigg]
		-\gamma^2 \sum_{s'}\bP^{\pi^{(0)}}(s'|s) \bv^2[\pi^{\infty(1)}](s') + \gamma^2 \sum_{s'}\bP^{\pi^{(0)}}(s'|s) \bv^2[\pi^{\infty(1)}](s')\\
		&= \gamma^2 \sum_{s'}\bP^{\pi^{(0)}}(s'|s) 
		\var\bigg[\sum_{t=0}^\infty \gamma^{t}\br(s^{t+1}, \pi^{(t+1)}(s))\bigg|s^1 = s'\bigg] + \gamma^2 \sum_{s'}\bP^{\pi^{(0)}}(s'|s) \bv^2[\pi^{\infty(1)}](s')
	\end{align*}
	Combining the above two equations, we have, 
	\begin{align*}
	\EE\bigg[\bigg(\sum_{t=0}^\infty \gamma^{t}\br(s^t, \pi^{(t)}(s^t))\bigg)^2\bigg|s^0=s\bigg]
	&= \br(s, \pi^{(0)}(s))^2 + 2\gamma \br(s, \pi^{(0)}(s)) \sum_{s'}\bP^{\pi^{(0)}}(s'|s)\bv[\pi^{\infty(1)}](s')
	\\
	&\quad+ \gamma^2 \sum_{s'}\bP^{\pi^{(0)}}(s'|s) 
	\var\bigg[\sum_{t=0}^\infty \gamma^{t}\br(s^{t+1}, \pi^{(t+1)}(s))\bigg|s^1 = s'\bigg]
	\\
	& \quad + \gamma^2 \sum_{s'}\bP^{\pi^{(0)}}(s'|s) \bv^2[\pi^{\infty(1)}](s')
	\end{align*}
	We thus obtain
	\begin{align*}
		\var&\bigg[\sum_{t=0}^\infty \gamma^{t}\br(s^t, \pi^{(t)}(s^t))\bigg|s^0 = s\bigg] 
=  \gamma^2 \sum_{s'}\bP^{\pi^{(0)}}(s'|s) 
\var\bigg[\sum_{t=0}^\infty \gamma^{t}\br(s^{t+1}, \pi^{(t+1)}(s))\bigg|s^1 = s'\bigg]\\
&\qquad+ \gamma^2 \sum_{s'}\bP^{\pi^{(0)}}(s'|s) \bv^2[\pi^{\infty(1)}](s') 
- \gamma^2 (\bP^{\pi^{(0)}} \bv[\pi^{\infty(1)}])^2(s)\\
&=	\gamma^2 \sum_{s'}\bP^{\pi^{(0)}}(s'|s) 
\var\bigg[\sum_{t=0}^\infty \gamma^{t}\br(s^{t+1}, \pi^{(t+1)}(s))\bigg|s^1 = s'\bigg] +  \gamma^2 \var({(\bv[\pi^{\infty(1)}]) })_{\pi^{(0)}}
	\end{align*}
	Let LHS and RHS be the left hand side and right hand side of \eqref{eqn:newvar} respectively.
	Then we have, 
	\begin{align*}
	LHS &= 
	\gamma^2 
	\sum_{s'}\bP^{\pi^{(0)}}(s'|s) \var\bigg[\sum_{t=0}^\infty \gamma^{t}\br(s^{t+1}, \pi^{(t+1)}(s))\bigg|s^1 = s'\bigg]
	+ \gamma^2 \var({(\bv[\pi^{\infty(1)}]) })_{\pi^{(0)}}\\
	&=\gamma^4 
	\sum_{s',s''}\bP^{\pi^{(0)}}(s'|s)\bP^{\pi^{(1)}}(s''|s') \var\bigg[\sum_{t=0}^\infty \gamma^{t}\br(s^{t+2}, \pi^{(t+2)}(s))\bigg|s^2 = s''\bigg] + \gamma^4 \bP^{\pi^{(0)}}\var({(\bv[\pi^{\infty(2)}])})_{\pi^{(1)}} \\
	 & \qquad+\gamma^2 \var({(\bv[\pi^{\infty(1)}])})_{\pi^{(0)}}.
	\end{align*}
	Applying the above equality recursively for $\var({(\bv[\pi^{\infty(i)}])})$ completes the proof.
\end{proof}

Based on the above two lemmas, we immediately obtain the following worst-case bound for the error accumulation.
\begin{corollary}
	\label{cor:error accur}
	\begin{align*}
	\sqrt{R\sum_{i=1}^{R}\gamma^{2(R-i)}  
		\bP^{\sigma_{\aux}^{(R)}} \cdot \bP^{\sigma_{\aux}^{(R-1)}} \cdot \ldots \bP^{\sigma_{\aux}^{(i+1)}}\var({\bv[\sigma_{\aux}^{\infty(i-1)}]})_{\sigma_{\aux}^{(i)}}}
	&\le \sqrt{\frac{R}{\gamma^2(1-\gamma)^2}}.
	\end{align*}
\end{corollary}
\begin{proof}
We use that 
\begin{align*}
\Big[\sum_{i=1}^{R}\gamma^{2(R-i)}  
\bP^{\sigma_{\aux}^{(R)}} \cdot \bP^{\sigma_{\aux}^{(R-1)}} \cdot &\ldots \bP^{\sigma_{\aux}^{(i+1)}}\var({\bv[\sigma_{\aux}^{\infty(i-1)}]})_{\sigma_{\aux}^{(R-i)}}\Big](s)\\
&\le \frac{1}{\gamma^2}\cdot  \var\bigg[\sum_{i=0}^\infty \gamma^{i}\br(s^{i},  \sigma_{\aux}^{(R-i)}(s))\bigg|s^0 = s\bigg].
\end{align*}
Since $\sum_{i=0}^\infty \gamma^{i}\br(s^i, \sigma_{\aux}^{(i)}(s))\le (1-\gamma)^{-1}$, we have
   $\var({\bv[\sigma_{\aux}^{\infty(i-1)}]})_{\sigma_{\aux}^{(R-i)}}\le (1-\gamma)^{-2}$ as desired.
\end{proof}

\subsubsection{Putting Everything Together}
\begin{proof}[Proof of Proposition~\ref{mainprop0}]
	By Corollary~\ref{corr:expansion}, we have, 
	\begin{align*}
	\bv^{(R)}- \bv^*\le \bv^{(R)}-\bv[\sigma_{\aux}^{(R)}]
	&\le 
	 \gamma^{R}\bP^{\sigma_{\aux}^{(R)}}\cdot \bP^{\sigma_{\aux}^{(R-1)}} \cdot \ldots \bP^{\sigma_{\aux}^{(1)}}\big(\bv^{(0)} - \bv^*\big)\\
	 &\qquad+\sum_{i=1}^{R}\gamma^{R-i}  
	 \bP^{\sigma_{\aux}^{(R)}} \cdot \bP^{\sigma_{\aux}^{(R-1)}} \cdot \ldots \bP^{\sigma_{\aux}^{(i+1)}}  \bepsilon^{(i)}_{\sigma_{\aux}^{(i)}}\\
	   &\le
	   u/4 + \underset{\textcircled{1}}{\underline{\sum_{i=1}^{R}\gamma^{R-i}  
	\bP^{\sigma_{\aux}^{(R)}} \cdot \bP^{\sigma_{\aux}^{(R-1)}} \cdot \ldots \bP^{\sigma_{\aux}^{(i+1)}}  \bepsilon^{(i)}_{\sigma_{\aux}^{(i)}}}},
	\end{align*}
	where the first inequality holds 
	for sufficiently large $R$.
	Consider the second term.
	Since
	\[
	\bepsilon^{(i)} = \sqrt{L \cdot\var({\bv^{(0)}}) /m} + \beta\cdot (L/m)^{3/4}+  u/ (CR).
	\]
	We bound
	\[
	\sum_{i=1}^{R}\gamma^{R-i}  
	\bP^{\sigma_{\aux}^{(R)}} \cdot \bP^{\sigma_{\aux}^{(R-1)}} \cdot \ldots \bP^{\sigma_{\aux}^{(i+1)}}  \cdot \beta\cdot (L/m)^{3/4} 
	\le R\beta\cdot (L/m)^{3/4} 
	\]
	and
	\[
	\sum_{i=1}^{R}\gamma^{R-i}  
	\bP^{\sigma_{\aux}^{(R)}} \cdot \bP^{\sigma_{\aux}^{(R-1)}} \cdot \ldots \bP^{\sigma_{\aux}^{(i+1)}}  \cdot u/(CR)
	\le Ru/(CR).
	\]
	We thus have,
	\[
	\textcircled{1}\le \sum_{i=1}^{R}\gamma^{R-i}  
	\bP^{\sigma_{\aux}^{(R)}} \cdot \bP^{\sigma_{\aux}^{(R-1)}} \cdot \ldots \bP^{\sigma_{\aux}^{(i+1)}}  \sqrt{L\cdot \var({\bv^{(0)}})_{\sigma_{\aux}^{(i)}} /m}
	+ R\beta\cdot(L/m)^{3/4} + R\cdot u / (CR).
	\]
	Note that
	\[
	\sqrt{\var(\bv^{(0)})_{\sigma_{\aux}^{(i)}}}
	\le \sqrt{\var(\bv^{\sigma_{\aux}^{(i-1)}})_{\sigma_{\aux}^{(i)}}} + 
	\|\bv[\sigma_{\aux}^{(i-1)}] - \bv^{(0)}\|_{\infty}. 
	\]
	Now consider
	\begin{align*}
	\|\bv[\sigma_{\aux}^{(i-1)}] - \bv^{(0)}\|_{\infty}
	&\le  \|\bv^{(0)} - \bv^{(R)}\|_{\infty}+\|\bv[\sigma_{\aux}^{(i-1)}] - \bv^{(i-1)}\|_{\infty}.
	\end{align*}
	We bound $ \|\bv^{(0)} - \bv^{(R)}\|_{\infty}\le u$.
	Applying Corollary~\ref{corr:expansion} again, we have
	\begin{align*}
		 \|\bv[\sigma_{\aux}^{(i-1)}] - \bv^{(i-1)}\|_{\infty}
		 &\le u/4 + \sum_{i=1}^{R}\gamma^{R-i}  
		 		\bP^{\sigma_{\aux}^{(R)}} \cdot \bP^{\sigma_{\aux}^{(R-1)}} \cdot \ldots \bP^{\sigma_{\aux}^{(i+1)}}  \bepsilon^{(i)}_{\sigma_{\aux}^{(i)}}\\
		& \le u/4 + \sum_{i=1}^{R}\gamma^{R-i}  
		\bP^{\sigma_{\aux}^{(R)}} \cdot \bP^{\sigma_{\aux}^{(R-1)}} \cdot \ldots \bP^{\sigma_{\aux}^{(i+1)}}  \sqrt{L\cdot \var({\bv^{(0)}})_{\sigma_{\aux}^{(i)}} /m}\\
		&\qquad
		+ R\beta\cdot(L/m)^{3/4} + R\cdot u / (CR) 		
	\end{align*}
	With a natural bound, $\var[\bv^{(0)}]\le \beta^2\cdot \one$, we have 
	\begin{align*}
		\|\bv[\sigma_{\aux}^{(i-1)}] - \bv^{(0)}\|_{\infty}\le u+
	\|\bv[\sigma_{\aux}^{(i-1)}] - \bv^{(i-1)}\|_{\infty}
	&\le R\beta\sqrt{L/m} +R\beta(L/m)^{3/4}+u/C'
	\end{align*}
	for some constant $C'>0$.
	Therefore, 
	\begin{align*}
	\bv^{(R)}- \bv^* &\le
	\sum_{i=1}^{R}\gamma^{R-i}  
	\bP^{\sigma_{\aux}^{(R)}} \cdot \bP^{\sigma_{\aux}^{(R-1)}} \cdot \ldots \bP^{\sigma_{\aux}^{(i+1)}}  \sqrt{\frac{L\var(\bv[{\sigma_{\aux}^{(i-1)}}])_{ \pi_{\aux}^{(i)}}}m}\\ & \qquad 
	+ R\bigg(R\beta\sqrt{\frac{L}{m}} +R\beta\Big(\frac{L}{m}\Big)^{3/4} + \frac{u}{C'}\bigg)\cdot\sqrt{\frac{L}{m}} + \frac{u}{4}
	 +\frac{u}{C}+ R\beta\cdot\Big(\frac{L}{m}\Big)^{3/4}\\
	&\le\sqrt{\frac{LR\beta^2}{\gamma^2m}} + R\bigg(R\beta\sqrt{\frac{L}{m}} +R\beta\Big(\frac{L}{m}\Big)^{3/4} + \frac{u}{C'}\bigg)\cdot\sqrt{\frac{L}{m}} + \frac{u}{4}
	+\frac{u}{C}+ R\beta\cdot\Big(\frac{L}{m}\Big)^{3/4}\\
	&\le u/2
	\end{align*}
	for large enough constants, $C$, and that in $m$. 
\end{proof}

\subsection{Algorithm that Computes a Monotone Sequence}
\label{sec:alg}
Here we show that 
Algorithm~\ref{alg-halfErr} or Algorithm~\ref{alg-halfErr-inc}
computes an MDVSS or MIVSS respectively. 
\begin{proposition}
	\label{mainprop}
	Let $u\in(0,\beta], \beta = (1-\gamma)^{-1}, \delta\in(0,1)$, and $R=\Theta[\beta\log(\beta/u)]$. Further, let $L = \Theta(\log[\delta^{-1}\beta R|\cS||\cA|])$, and	$m=\Omega(\beta^3\cdot\max(u^{-2}, 1)\cdot{\log(|\cS||\cA|\delta^{-1})})$.
	Then there exists an algorithm, on input a stochastic game with a generative model 
	with a sampling oracle, $\cG =(\cS:=\cS_{\min}\cup \cS_{\max}, \bP, \br, \gamma)$, a value-strategy pair $(\pi^{(0)}, \bv^{(0)})$ satisfying $\cT_{\pi^{(0)}} [\bv^{(0)}] \le \bv^{(0)}, \cT [\bv^{(0)}]\le \bv^{(0)}$ (or $\cT_{\pi^{(0)}} [\bv^{(0)}] \ge \bv^{(0)}, \cT [\bv^{(0)}]\ge \bv^{(0)}$), and $\|\bv^{(0)} - \bv^*\|_{\infty}\le u$ for some $u>0$, outputs, with probability at least $1-\delta$, an MDVSS (or MIVSS)  $\{\bv^{(i)},\bQ^{(i)}, \pi^{(i)}, \bepsilon^{(i)}\}_{i=0}^R$ 
	by querying 
	\[
	Z = O\big[|\cS||\cA|\cdot (m + R\beta^{2}\log[R|\cS||\cA|\delta^{-1}])\big]
	\]
	samples, where
	\[
	\bepsilon^{(i)} = \sqrt{L \bsigma_{\bv^{(0)}} /m} + \beta\cdot (L/m)^{3/4}+  u/ (CR),
	\]
	for some large constant $C>1$ and $\bsigma_{\bv^{(0)}} := \var[\bv^{(0)}]$
	is the variance vector for vector $\bv^{(0)}$.
	The algorithms uses space $O(|\cS||\cA|)$ and halts in time $O(Z)$.
\end{proposition}
This section is devoted to proving Proposition~\ref{mainprop}.
The algorithm of obtaining MDVSS and MIVSS is provided in Algorithm~\ref{alg-halfErr} and \ref{alg-halfErr-inc}.
\paragraph{The Good Events}
Suppose we are given an arbitrary input vector $\bv^{-(0)}, \bv^{+(0)}\in[0,(1-\gamma)^{-1}]^{\cS}$ with $\bv^*-u\one\le\bv^{-(0)}\le \bv^* \le  \bv^{+(0)} \le \bv^*+u\one$, $ \bv^{-(0)}\le \cT  [\bv^{-(0)}]$, $\bv^{-(0)}\le \cT_{\pi^-}  [\bv^{-(0)}]$, 
$ \cT[\bv^{+(0)}]\le   \bv^{+(0)}$ and $\cT_{\pi^+} [\bv^{+(0)}]\le  \bv^{+(0)}$.
Since the algorithm is randomized,
to begin our analysis, we define a sequence of events for the iterates.
We will show that these events happen with high probability via concentration inequalities.
\begin{definition}
\label{def:events}
Let $\wt{\bw}^-$ and $\wt{\bw}^+$ be the estimate defined in Line~\ref{alg1: compute w1+} (of Algorithm~\ref{alg-halfErr} or Algorithm~\ref{alg-halfErr-inc} respectively).
Denote $\alpha_1\gets L/m_1\le 1$.
Let $\cE_0$ be the event that
\begin{align}
|\wt{\bw}^{\pm}-\bP \bv^{\pm(0)}|\le \sqrt{\alpha_1 \bsigma_{\bv^{\pm(0)}}} + \alpha_1^{3/4}\cdot\|\bv^{\pm(0)}\|_{\infty}\cdot \one.
\end{align}
For each $i> 0$,
let $\wt{\bg}^{\pm(i)}$ be given in Line~\ref{alg1: compute g}. 
Let $\cE_i$ be the event that
\begin{align}
|\wt{\bg}^{\pm(i)}  - \bP[\bv^{\pm(i)} - \bv^{\pm (0)}] |\le C(1-\gamma)u\cdot \one.
\end{align}
for some sufficiently small constant $C>0$.
\end{definition}
\begin{lemma}
	\label{lemma:base event}
	For some sufficiently large constant $c$ in $L$, 
	$\Pr[\cE_0]\ge 1-O(\delta/R)$.
\end{lemma}
\begin{proof}
	Note that $\|v^{\pm(0)}\|_{\infty}\le (1-\gamma)^{-1}$.
	By a straightforward application of a Hoeffding bound and Bernstein inequality and a union bound over all $(s,a)$, we reach the desired inequality.
	More details can be find in the proof of Lemma~5.1 in \cite{sidford2018near}.
\end{proof}
\paragraph{The Implications of the Good Events}
We now illustrate the consequences of these good events.

\begin{lemma}[Implications of $\cE_0$]
\label{lemma:e0}
On $\cE_0$, we have, for all $(s,a)\in \cS\times\cA$,
\begin{align*}
0\le \br(s,a) +\gamma \bP(\cdot|s,a)^\top \bv^{-(0)} - \bQ^{-(0)}(s,a) &\le 2\sqrt{\alpha_1 \bsigma_{\bv^{-(0)}}} + 
2\alpha_1^{3/4}\|\bv^{-(0)}\|_{\infty}\quad\text{and}\quad\\
0\le \bQ^{-(0)}(s,a) &\le \bQ^*(s,a),
\end{align*}
and
\begin{align*}
0\le  \bQ^{+(0)}(s,a) -r(s,a) -\gamma \bP(\cdot|s,a)^\top \bv^{+(0)} &\le 2\sqrt{\alpha_1 \bsigma_{\bv^{+(0)}}} + 
2\alpha_1^{3/4}\|\bv^{+(0)}\|_{\infty}\quad\text{and}\quad\\
\bQ^*(s,a)\le \bQ^{+(0)}(s,a)&\le \beta,
\end{align*}
where $\alpha_1= L/m_1$.
\end{lemma}
\begin{proof}
We prove the first inequality and the second inequality follows similarly.
Condition on $\cE_0$, we have
\[
|\br+ \gamma \wt{\bw}^{-} - \br-\gamma \bP \bv^{-(0)}|\le  \sqrt{\alpha_1 \bsigma_{\bv^{-(0)}}} + 
\alpha_1^{3/4}\|\bv^{-(0)}\|_{\infty}.
\]
Since 
\[
\bQ^{-(0)} = \max\Big[\br+ \gamma \wt{\bw}^- - \sqrt{\alpha_1 \bsigma_{\bv^{-(0)}}} - 
\alpha_1^{3/4}\|\bv^{-(0)}\|_{\infty}, {\bf0}\Big],
\]
we have
\[
0\le \br +\gamma \bP \bv^{-(0)} - \bQ^{-(0)} \le  2\sqrt{\alpha_1 \bsigma_{\bv^{-(0)}}} +
2\alpha_1^{3/4}\|\bv^{-(0)}\|_{\infty}.
\]
Moreover, since $v^{(0)}\le v^*$, we have
\[
\bQ^{-(0)}\le  r(s,a) +\gamma \bP(\cdot|s,a)^\top \bv^{-(0)} \le \bQ^{*},
\]
completing the proof.
\end{proof}
\begin{lemma}[Implications of $\cE_i$,  (1)]
\label{lemma:e11}
Then for any $i> 0$, 
conditioning on $\cE_0, \cE_1, \ldots, \cE_{R}$, we have
$\{\bv^{-(i)}, \bQ^{-(i)}, \pi^{-(i)}, \bxi^{-(i)} \}_{i=0}^{R}$ is an MIVSS  and
$\{\bv^{+(i)}, \bQ^{+(i)}, \pi^{+(i)}, \bxi^{+(i)} \}_{i=0}^{R}$ is an MDVSS
where 
\[
\bxi^{\pm(i)} = 2\sqrt{\alpha_1 \bsigma_{\bv^{\pm(0)}}} +
2[\alpha_1^{3/4}\beta + C(1-\gamma)u]\cdot \one
\]
for some sufficiently small $C>0$.
\end{lemma}
\begin{proof}
We prove the first part of the lemma, i.e., $\{\bv^{-(i)}, \bQ^{-(i)}, \pi^{-(i)}, \bxi^{-(i)} \}_{i=0}^{R}$ is an MIVSS. Then the second part follows similarly.
It is clear from the definition (Line~\ref{alg: v2-}) that
\[
\bv^{-(0)}\le \bv^{-(1)} \ldots \le \bv^{-(R)}. 
\]
To prove property 1 of MIVSS, we need additionally to show 
\[
\bv^{-(R)}\le \bv^*.
\]
This follows if property 2, i.e.,  
\[
\forall i\in [0, R]:\quad \bv^{-(i)}\le \cT_{\pi^{-(i)}} [\bv^{-(i)}], \quad \bv^{-(i)}\le \cT [\bv^{-(i)}].
\]
Indeed,
\[
\bv^{-(i)}\le \cT [\bv^{-(i)}] \le \cT^2 [\bv^{-(i)}]  \ldots \le  \cT^{\infty} [\bv^{-(i)}] = \bv^*.
\]
We now prove property 2 by induction on $i$.
It immediately follows from the initial condition that
\[
\bv^{-(0)}\le \cT_{\pi^{-(0)}} [\bv^{-(0)}], \quad \bv^{-(0)}\le \cT [\bv^{-(0)}].
\]
Suppose this property holds for all $0,1,\ldots, i-1$ for some $i>1$. 
We now consider the case $i$.
Let $\wt{\bQ}^{-(i)} = \br + \gamma \bP \bv^{-(i-1)}$.
Since $\|\wt{\bg}^{-(i)} - \bP(\bv^{-(i-1)} - \bv^{-(0)})\|_{\infty}\le C(1-\gamma)\epsilon$, we have 
\begin{align}
\label{eqn:q-bound}
|\wt{\bQ}^{-(i)} -\br - \gamma \bw^{-(i-1)} - \gamma \bg^{-(i-1)} | \le \bxi^{-(i)}/2.
\end{align}
Since
\begin{align*}
\bQ^{-(i)}(s,a) &y= \max\Big[\br(s,a) + \gamma \bw^{-(i-1)}(s,a) + \gamma \bg^{-(i-1)}(s,a), 0\Big] \le \br(s,a) + \gamma   \bP(\cdot|s,a)^\top \bv^{-(i-1)} \\
&= \wt{\bQ}^{-(i)}(s,a)
\end{align*}
we have, for any $s\in \cS$
\begin{align*}
\min_a \bQ^{-(i)}(s,a)\le\min_a \wt{\bQ}^{-(i)}(s,a)
\quad\text{and}\quad \max_a \bQ^{-(i)}(s,a)\le\max_a \wt{\bQ}^{-(i)}(s,a).
\end{align*}
For each $s\in \cS$, denote 
\[
\wt{\bv}^{-(i)}(s)= \bQ^{-(i)}(s, \wt{\pi}^{-(i)}(s)),
\]
where $\wt{\pi}^{-(i)}$ is given in Line~\ref{alg: pit} (that achieves $\max_{a} \bQ^{-(i)}(s,a)$ or $\min_{a} \bQ^{-(i)}(s,a)$).
To show $\bv^{-(i)}\le \cT [\bv^{-(i)}]$ and $v^{-(i)}\le \cT_{\pi^{-(i)}} [\bv^{-(i)}]$, we do a case analysis for a state in $\cS$.
Firstly, we consider state $s\in \cS_{\min}$. 
For each state $s\in \cS_{\min}$, note that $\wt{\pi}^{-(i)}(s) := \arg\min_a \bQ^{-(i)}(s,a)$.
By Line~\ref{alg: v2-}, ${\bv}^{-(i)}(s)$ and $\pi^{-(i)}(s)$ have the following two possibilities,
\begin{enumerate}
\item $\bv^{-(i-1)}(s) \le    \wt{\bv}^{-(i)}(s)$ $\Rightarrow$ ${\bv}^{-(i)}(s) = \wt{\bv}^{-(i)}(s)$ and $\pi^{-(i)}(s) = \wt{\pi}^{-(i)}(s)$;
\item $\bv^{-(i-1)}(s) >   \wt{\bv}^{-(i)}(s)$ $\Rightarrow$  ${\bv}^{-(i)}(s) ={\bv}^{-(i-1)}(s)$ and 
${\pi}^{-(i)}(s) ={\pi}^{-(i-1)}(s)$.
\end{enumerate}
Considering case 1.,
we have,
\begin{align*}
\bv^{-(i)}(s) &= \bQ^{-(i-1)}(s,\pi^{-(i)}(s))\le \min_a \wt{\bQ}^{-(i)}(s,a) = \cT [\bv^{-(i-1)}](s) \le 
\cT [\bv^{-(i)}](s)
\quad\text{and} \\ 
\bv^{-(i)}(s) &= \bQ^{-(i-1)}(s,\pi^{-(i)}(s))\le  \wt{\bQ}^{-(i)}(s,\pi^{-(i)}(s))
= \cT_{\pi^{-(i)}} [\bv^{-(i-1)}](s)
\le \cT_{\pi^{-(i)}}[ \bv^{-(i)}](s).
\end{align*}

Considering case 2., we have, by induction hypothesis ${\bv}^{-(i-1)}(s)\le \cT [{\bv}^{-(i-1)}](s)$, ${\bv}^{-(i-1)}(s)\le \cT_{{\pi}^{-(i-1)}}[{\bv}^{-(i-1)}](s)$.
Thus 
\begin{align*}
{\bv}^{-(i)}(s) &\le \cT[{\bv}^{-(i-1)}](s)\le  \cT[{\bv}^{-(i)}](s)~\text{,}\\
{\bv}^{-(i)}(s)&\le \cT_{{\pi}^{(i-1)}}[{\bv}^{-(i-1)}](s)\le  \cT_{{\pi}^{-(i-1)}}[{\bv}^{-(i)}](s) 
=\cT_{{\pi}^{-(i)}}[{\bv}^{-(i)}](s).
\end{align*}
It follows similarly for the case of $s\in \cS_{\max}$ (by just replacing the $\min$ by $\max$ in the above argument).
This completes the induction step and hence the property 2.

We now prove property 3 and 4.
By Equation~\ref{eqn:q-bound}, we immediately have,
\begin{align*}
\bQ^{-(i)}\ge  \br + \gamma   \bP(\bv^{-(i)}) - \bxi^{-(i)}
\end{align*}
proving property 3.
Lastly, by Line~\ref{alg: v2-}, we have
\[
v^{-(i)}\ge \cT[\bQ^{-(i)}],
\]
completing the proof of property 4 and the lemma.
\end{proof}

\paragraph{The Probability of Good Events}
Note that the random samples in the successive improvement phase are independent with event $\cE_0$. We then have the following lemma.

\begin{lemma}
\label{lemma:accurate_estimate}
Suppose the algorithm does not halt at iteration $i\ge 1$, then,
\[
\Pr[\cE_{i}|\cE_0, \cE_1,\ldots, \cE_{i-1}] \ge 1-O(\delta/R).
\]
\end{lemma}
\begin{proof}
On $\cE_0, \cE_1, \ldots, \cE_{i-1}$, 
we have $\bv^{- (0)} \le \bv^{-(i)}\le \bv^*\le \bv^{+(i)}\le \bv^{+(0)}$, thus $\|\bv^{\pm (i)} - \bv^{\pm (0)}\|\le u$.
Applying Hoeffding bound, Bernstein's inequality and a union bound over all $(s, a)$,  we have that with probability at least $1-O(\delta/R)$, 
$\|g^{(i)} - P[v^{(i)} - \bv^{(0)}]\|_{\infty}\le (1-\gamma)\epsilon/32$, completing the proof.
\end{proof}

Therefore, we have the following lemma.
\begin{lemma}
\label{lemma:all good events happend prob}
Let $R = \lceil c_1\beta\ln[\beta \epsilon^{-1}]\rceil$ be an integer for some constant $c_1$. 
Then, with probability at least $1-O(\delta)$, $\cE_0, \{\cE_{i}\}_{i=1}^{R}$ all happen.
\end{lemma}
\begin{proof}
By a straightforward calculation, we have \\
$
\Pr[\cap_{i=0}^R\cE_{i}] = \Pr[\cE_{0}]\Pr[\cE_1|\cE_{0}]\Pr[\cE_2|\cE_{0}, \cE_1] \ldots \Pr[\cE_{R}|\cE_{0}, \cE_1, \ldots, \cE_{R-1}]
\ge 1-O(\delta).
$
\end{proof}

\paragraph{Putting It Together}
\begin{proof}[Proof of Proposition~\ref{mainprop}]
Let $\cE = \cap_{i=0}^R\cE_{i}$,
we have shown, on $\cE$, the outputs of Algorithm~\ref{alg-halfErr} and \ref{alg-halfErr-inc}, $\{\bv^{-(i)}, \bQ^{-(i)}, \pi^{-(i)}, \bxi^{-(i)} \}_{i=0}^{R}$ is an MIVSS  and
$\{\bv^{+(i)}, \bQ^{+(i)}, \pi^{+(i)}, \bxi^{+(i)} \}_{i=0}^{R}$ is an MDVSS.
By the above lemmas, $\cE$ happens with probability at least $1-\Theta(\delta)$.
This completes the proof of the proposition.
\end{proof}

\begin{proof}[Proof of Theorem~\ref{mainthm}]
The theorem is proved by combining Proposition~\ref{prop:main0}, \ref{mainprop0},  and \ref{mainprop}. 
\end{proof}

\begin{algorithm}[htb!]
	\caption{
		QVI-MIVSS:
		\label{alg-halfErr-inc} 
		algorithm for computing monotone increasing value-strategy sequences.
	}
	{\small
		\begin{algorithmic}[1]
			\State 
			\textbf{Input:} A generative model for stochastic game $\cM=(\cS, \cA, \br, \bP, \gamma)$;
			\State
			\textbf{Input:} Precision parameter $u\in[0,(1-\gamma)^{-1}]$; 
			and error probability $\delta \in (0, 1)$
			\State 
			\textbf{Input:}
			Initial values $\bv^{-(0)}, \pi^{-(0)}$ that satisfies monotonicity:
			{\small
				\vspace{-2mm}
				\begin{align*}
					&\bv^*-u\one\le \bv^{-(0)} \le \bv^*, \quad \bv^{-(0)}\le \cT \bv^{-(0)}, \quad\text{and}\quad \bv^{-(0)}\le \cT_{\pi^{-(0)}} \bv^{-(0)};
				\end{align*}
				\vspace{-4mm}
			}
			\State\textbf{Output:} $\{\bv^{-(i)}, \bQ^{-(i)}, \pi^{-(i)}, \bxi^{- (i)}\}_{i=0}^{R}$  which is an MIVSS  with high probability.
			\State
			\State\textbf{INITIALIZATION:} 
			\State Let $c_1, c_2, c_3, c$ be some tunable absolute constants;
			\State \textcolor{OliveGreen}{\emph{\textbackslash \textbackslash Initialize constants:}}
			\State \qquad$\beta\gets (1-\gamma)^{-1}$, and $R\gets\lceil c_1\beta\ln[\beta u^{-1}]\rceil$; 
			\State \qquad$m_1 \gets{c_2\beta^3\cdot\min(1,u^{-2})\cdot{\log(8|\cS||\cA|\delta^{-1})} }{}$; 
			\State  \qquad$m_2\gets {c_3\beta^{2}\log[2R|\cS||\cA|\delta^{-1}]}$;
			\State\qquad $\alpha_1\gets L/m_1$ where $L = c\log(|\cS||\cA|\delta^{-1}(1-\gamma)^{-1}\epsilon^{-1})$; 
			\State
			\textcolor{OliveGreen}{\emph{\textbackslash \textbackslash Obtain an initial batch of samples:}}r
			\State For each  $(s, a)\in \cS\times\cA$:
			obtain independent samples $s_{s,a}^{(1)}, s_{s,a}^{(2)}, \ldots, s_{s,a}^{(m_1)}$ from $\bP(\cdot | s,a)$;
			\State
			Initialize: $~\bw^{-}=\wt{\bw}^- = \wh{\bsigma}^-=\bQ^{-(0)}=\bQ^{-(1)} \gets {\bf0}_{\cS \times \cA}$ and $i\gets 0$;	
			\For{each $(s, a)\in \cS\times\cA$} 
			\State \textcolor{OliveGreen}{\emph{\textbackslash \textbackslash Compute empirical estimates of $\bP_{s,a}^{\top}\bv^{-(0)}$ and $\var({\bv^{-(0)}})(s,a)$:}}
			\State 		\label{alg1: compute w1-} 
			$\wt{\bw}^{-}(s,a) \gets \frac{1}{m_1} 
			\sum_{j=1}^{m_1} \bv^{- (0)}(s_{s,a}^{(j)})$;
			
			\State  $\wh{\bsigma}^{-}(s,a)\gets \frac{1}{m_1} \sum_{j=1}^{m_1}(\bv^{- (0)})^2(s_{s,a}^{(j)}) - (\wt{\bw}^{-})^2(s,a)$ ;
			
			\State \textcolor{OliveGreen}{\emph{\textbackslash \textbackslash Shift the empirical estimate to have one-sided error and guarantee monotonicity:}} 
			\State 	
			$\bw^-(s, a) \gets \wt{\bw}^-(s,a) - \sqrt{\alpha_1\wh{\bsigma}^-(s,a)} - \alpha_1^{3/4}\beta$;
			\State \textcolor{OliveGreen}{\emph{\textbackslash \textbackslash Compute coarse estimate of the  $Q$-function and make sure its value is in $[0,\beta]$:}}
			\State
			$\bQ^{-(1)}(s,a) \gets \clip[\br(s,a) + \gamma \bw^{-}(s,a), 0, \beta]$
			
			\EndFor
			\State
			\State\textbf{REPEAT:} \textcolor{OliveGreen}{\emph{\qquad\qquad\textbackslash \textbackslash successively improve}}
			\For{$i=1$ to $R$}
			\State  \textcolor{OliveGreen}{\emph{\textbackslash \textbackslash Compute
					the one-step dynamic programming:}}
			\State\label{alg: pit-} Let ${\bv}^{-(i)} \gets \wt{\bv}^{-(i)}\gets \cT \bQ^{-(i-1)}$, ${\pi}^{-(i)}\gets\wt{\pi}^{-(i)}\gets \pi(\bQ^{-(i-1)})$;
			\State \textcolor{OliveGreen}{\emph{\textbackslash \textbackslash Compute strategy and value and maintain monotonicity:}} 
			\State For {each $s\in \cS$}:
			\State\qquad \label{alg: v2-} if ${\bv}^{-(i)}(s)\le \bv^{-(i-1)}(s)$, then 
			$\bv^{-(i)}(s)\gets \bv^{-(i-1)}(s)$ and $\pi^{-(i)}(s)\gets \pi^{-(i-1)}(s)$;
			\State\textcolor{OliveGreen}{\emph{\textbackslash \textbackslash Obtaining a small batch of samples:}}
			\State For each $(s, a)\in \cS\times\cA$:
			draw independent samples $\wt{s}_{s,a}^{(1)}, \wt{s}_{s,a}^{(2)}, \ldots, \wt{s}_{s,a}^{(m_2)}$ from $\bP(\cdot | {s,a})$;
			\State \textcolor{OliveGreen}{\emph{\textbackslash \textbackslash Compute the expected value,  $\bg^{\pm(i)}$, the estimate of $\bP \big[\bv^{\pm(i)}  - \bv^{\pm (0)}\big]$ with one-sided error:}}
			\State  \label{alg1: compute g-} Let $\wt{\bg}^{-(i)}(s,a)\gets {\frac{1}{m_2}} \sum_{j=1}^{m_2} \big[\bv^{- (i)}(\wt{s}_{s,a}^{(j)}) - \bv^{-(0)}(\wt{s}_{s,a}^{(j)}) \big]$;
			\State Let ${\bg}^{-(i)}(s,a)\gets\wt{\bg}^{-(i)}(s,a)- C(1-\gamma)u$, where $C$ is sufficiently small;
			\State \textcolor{OliveGreen}{\emph{\textbackslash \textbackslash Estimate the approximation error:}}
			\State $\bxi^{- (i)}\gets 2\sqrt{\alpha_1 \bsigma_{\bv^{- (0)}}} +
			2[\alpha_1^{3/4}\beta + C(1-\gamma)u]\cdot \one$
			\State  \textcolor{OliveGreen}{\emph{\textbackslash \textbackslash Improve $\bQ^{- (i)}$  make sure its value is in $[0,\beta]$:}}
			\State \label{alg: q-func-} $\bQ^{- (i+1)}\gets \clip\Big[\br + \gamma\cdot[\bw^{-}+\bg^{- (i)}],0, \beta\Big]$;
			\EndFor
			\State \textbf{return} 
			$\{\bv^{-(i)}, \bQ^{-(i)}, \pi^{-(i)}, \bxi^{- (i)}\}_{i=0}^{R}$
		\end{algorithmic}
	}
\end{algorithm}

\bibliographystyle{apalike}
\bibliography{ref,reference}

\clearpage
\appendix

%
%
%

\newpage
\appendix

%
%
%
%
%

\end{document}